\documentclass[11pt]{article}
\pdfoutput=1
\usepackage{wrapfig}
\usepackage[utf8]{inputenc} 
\usepackage[T1]{fontenc}    

\usepackage{graphicx}
\usepackage{subfigure}
\usepackage[colorlinks,
            linkcolor=red,
            anchorcolor=blue,
            citecolor=blue
            ]{hyperref}
\usepackage{url}            
\usepackage{booktabs}       
\usepackage{amsfonts}       
\usepackage{nicefrac}       
\usepackage{microtype}      
\usepackage{xcolor,colortbl}         
\usepackage{wrapfig}
\usepackage{caption}
\usepackage{enumitem}
\usepackage{tabularx}

\usepackage{mylatexstyle}
\usepackage[compact]{titlesec}

\usepackage{setspace}
\usepackage{fullpage}

\usepackage{todonotes}
\usepackage{colortbl}
\definecolor{LightCyan}{rgb}{0.8, 0.9, 1}

\allowdisplaybreaks

\def\cN{\mathcal{N}}

\newcommand{\method}{\texttt{SPIN}}

\definecolor{ao}{rgb}{0.0, 0.5, 0.0}

\title{\huge Self-Play Fine-Tuning of Diffusion Models for Text-to-Image Generation}

\author{
    Huizhuo Yuan\thanks{Equal contribution} \thanks{Department of Computer Science, University of California, Los Angeles, CA 90095, USA; e-mail: {\tt hzyuan@cs.ucla.edu}}
    \and
    Zixiang Chen\footnotemark[1] \thanks{Department of Computer Science, University of California, Los Angeles, CA 90095, USA; e-mail: {\tt chenzx19@cs.ucla.edu}}
    \and
    Kaixuan Ji\footnotemark[1] \thanks{Department of Computer Science, University of California, Los Angeles, CA 90095, USA; e-mail: {\tt kaixuanji@cs.ucla.edu}}
    \and
    Quanquan Gu\thanks{Department of Computer Science, University of California, Los Angeles, CA 90095, USA; e-mail: {\tt qgu@cs.ucla.edu}}
}

\begin{document}
\date{}
\maketitle
\begin{figure*}[ht!]
    \centering
\includegraphics[width=1.0\textwidth]{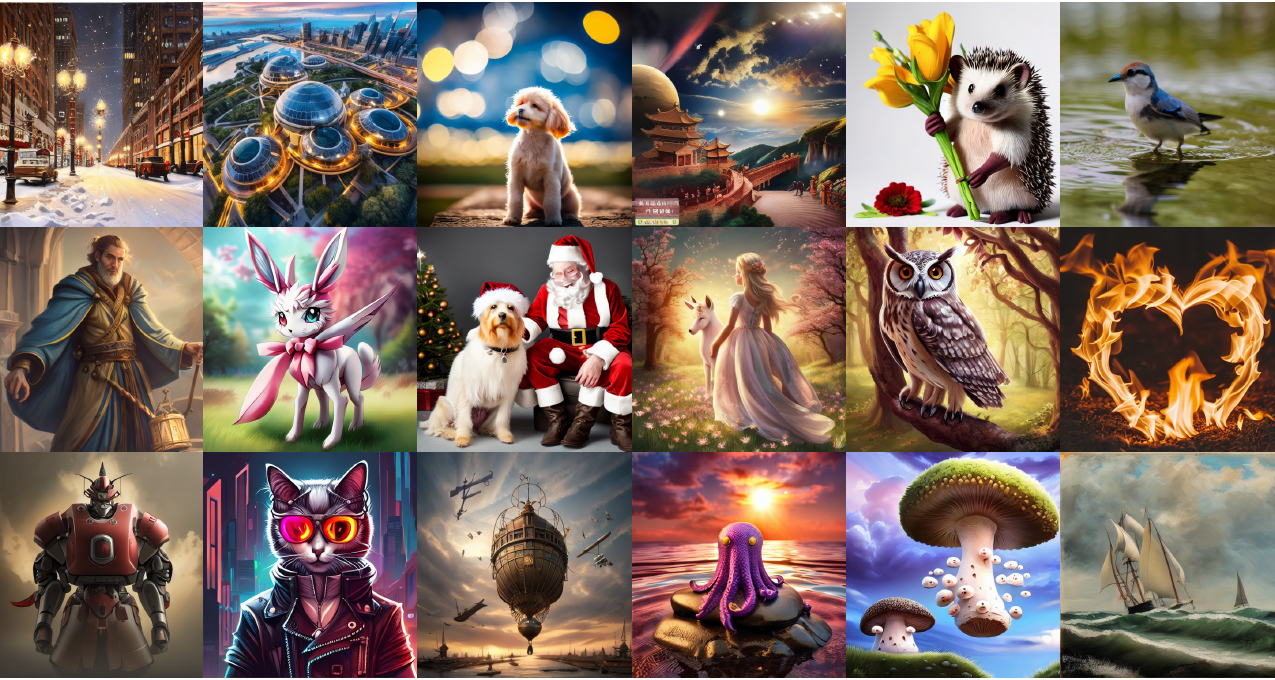}
\caption{We introduce SPIN-Diffusion, a self-play fine-tuning algorithm for diffusion models. The results are fine-tuned from Stable Diffusion v1.5 on the winner images of the Pick-a-Pic dataset. The prompts used for generating the above images are chosen from the Pick-a-Pic test set. The generated images demonstrate superior performance in terms of overall visual attractiveness and coherence with the prompts. SPIN-Diffusion is featured by its independence from paired human preference data, offering a useful tool for fine-tuning on custom datasets with only single image per text prompt provided.}
\end{figure*}

\newpage
\begin{abstract}%
Fine-tuning Diffusion Models remains an underexplored frontier in generative artificial intelligence (GenAI), especially when compared with the remarkable progress made in fine-tuning Large Language Models (LLMs). While cutting-edge diffusion models such as Stable Diffusion (SD) and SDXL rely on supervised fine-tuning, their performance inevitably plateaus after seeing a certain volume of data. Recently, reinforcement learning (RL) has been employed to fine-tune diffusion models with human preference data, but it requires at least two images (``winner'' and ``loser'' images) for each text prompt.
In this paper, we introduce an innovative technique called self-play fine-tuning for diffusion models (SPIN-Diffusion), where the diffusion model engages in competition with its earlier versions, facilitating an iterative self-improvement process. Our approach offers an alternative to conventional supervised fine-tuning and RL strategies, significantly improving both model performance and alignment. Our experiments on the Pick-a-Pic dataset reveal that SPIN-Diffusion outperforms the existing supervised fine-tuning method in aspects of human preference alignment and visual appeal right from its first iteration. By the second iteration, it exceeds the performance of RLHF-based methods across all metrics, achieving these results with less data.
\end{abstract}

\section{Introduction}
Diffusion models~\citep{ho2020denoising, peebles2023scalable, podell2023sdxl, nichol2021glide, rombach2022high, song2020denoising} have rapidly emerged as critical entities within the realm of generative AIs~\citep{creswell2018generative, kingma2013auto}, demonstrating exceptional capabilities in generating high-fidelity outputs. Their versatility spans a diverse area of applications, ranging from image generation~\citep{rombach2022high, podell2023sdxl, ramesh2022hierarchical} to more complex tasks like structure-based drug design~\citep{corso2022diffdock,guan2023decompdiff}, protein structure prediction~\citep{watson2021learning}, text generation~\citep{austin2021structured, zheng2023reparameterized,chen2023fast}, and more. Prominent diffusion models in image generation, including DALL-E~\citep{ramesh2022hierarchical}, Stable Diffusion~\citep{Rombach_2022_CVPR}, SDXL~\citep{podell2023sdxl}, and Dreamlike, etc., typically undergo a fine-tuning process following their initial pre-training phase.

Recently, using Reinforcement Learning (RL) for fine-tuning diffusion models has received increasing attention. 
~\citet{lee2023aligning} first studied the alignment of text-image diffusion models to human preferences using reward-weighted likelihood maximization with a reward function trained on human preference data. \citet{black2023training} formulated the fine-tuning of diffusion models as a RL problem solved by policy gradient optimization. In a concurrent work, \citet{fan2023dpok} studied a similar formulation but with a KL regularization.
Very recently,~\citet{wallace2023diffusion} have bypassed the need for training reward functions by using Direct Preference Optimization (DPO) \citep{rafailov2023direct} for fine-tuning diffusion models. Similar approach was proposed in \citet{yang2023using} as well.

While RL fine-tuning of diffusion methods has been proven effective, its dependency on human preference data, often necessitating multiple images per prompt, poses a significant challenge. 
In many datasets including the community-sourced ones featuring custom content, it is often the case to have only one image associated with each prompt. 
This makes RL fine-tuning infeasible. 


In this paper, drawing inspiration from the recently proposed self-play fine-tuning (SPIN) technique \citep{chen2024self} for large language models (LLM), we introduce a new supervised fine-tuning (SFT) method for diffusion models, eliminating the necessity for human preference data in the fine-tuning process.
Central to our method is a general-sum minimax game,
where both the participating players, namely the main player and the opponent player, are diffusion models.
The main player's goal is to discern between samples drawn from the target data distribution and those generated by the opponent player.
The opponent player's goal is to garner the highest score possible, as assessed by the main player.
A self-play mechanism can be made possible, if and only if the main player and the opponent player have the same structure, and therefore the opponent player can be designed to be previous copies of the main player~\citep{chen2024self}.


When applying the self-play fine-tuning technique~\citep{chen2024self} to diffusion models, there are two challenges: (a) an exponential or even infinite number of possible trajectories can lead to the same image. The generator in a diffusion model operates through a sequence of intermediate steps, but the performance of the generator is only determined by the quality of the image in the last step; and
(b) diffusion models are parameterized by a sequence of score functions, which are the gradient of the probabilities rather than probabilities in LLMs. Our algorithm design effectively surmounts these challenges by (a) designing an objective function that considers all intermediate images generated during the reverse sampling process; and (b) decomposing and approximating the probability function step-by-step into products related to the score function.  We also employ the Gaussian reparameterization technique in DDIM \citep{song2020denoising} to support the advanced sampling method. All these techniques together lead to an unbiased objective function that can be effectively calculated based on intermediate samples. For computational efficiency,  we further propose an approximate objective function, which eliminates the need for intermediate images used in our model. We call our algorithm SPIN-Diffusion.

\noindent\textbf{Contributions.} Our contributions are summarized below:
\begin{itemize}[leftmargin=*]
\setlength\itemsep{-0.5em}
    \item We propose a novel fine-tuning method for diffusion models based on the self-play mechanism, called SPIN-Diffusion. The proposed algorithm iteratively improves upon a diffusion model until converging to the target distribution. Theoretically, we prove that the model obtained by SPIN-Diffusion cannot be further improved via standard SFT. Moreover, the stationary point of our self-play mechanism is achieved when the diffusion model aligns with the target distribution. 

    \item Empirically, we evaluate the performance of SPIN-Diffusion on text-to-image generation tasks~\citep{ramesh2022hierarchical, rombach2022high, saharia2022photorealistic}. Our experiments on the Pick-a-Pic dataset~\citep{kirstain2023pick}, with base model being Stable Diffusion-1.5~\citep{Rombach_2022_CVPR}, demonstrate that SPIN-Diffusion surpasses SFT from the very first iteration. Notably, by the second iteration, SPIN-Diffusion outperforms Diffusion-DPO~\citep{wallace2023diffusion}
    that utilizes additional data from `loser' samples.
By the third iteration, the images produced by SPIN-Diffusion achieve a higher PickScore~\citep{kirstain2023pick} than the base model SD-1.5 $79.8\%$ of the times, and a superior Aesthetic score $88.4\%$ of the times.
    
\end{itemize}
SPIN-Diffusion exhibits a remarkable performance improvement over current state-of-the-art fine-tuning algorithms, retaining this advantage even against models trained with more extensive data usage. This highlights its exceptional efficiency in dataset utilization. It is beneficial for the general public, particularly those with restricted access to datasets containing multiple images per prompt.


\noindent\textbf{Notation.}
We use lowercase letters and lowercase boldface letters to denote scalars and vectors, respectively. We use $0:T$ to denote the index set $\{0, \dots, T\}$. In the function space, let $\cF$ be the function class. We use the symbol $\qb$ to denote the real distribution in a diffusion process, while $\pb_{\btheta}$ represents the distribution parameterized by a nueral network during sampling. The Gaussian distribution is represented as $\cN(\bmu, \bSigma)$, where $\bmu$ and $\bSigma$ are the mean and covariance matrix, respectively. Lastly, $\mathrm{Uniform}\{1,\ldots,T\}$ denotes the uniform distribution over the set $\{1, \ldots, T\}$.

\section{Related Work}
\noindent\textbf{Diffusion Models.}
Diffusion-based generative models \citep{sohl2015deep} have recently gained prominence, attributed to their ability to produce high-quality and diverse samples. A popular diffusion model is
denoising diffusion probabilistic modeling (DDPM) \citep{ho2020denoising}. \citet{song2020denoising} proposed a denoising diffusion implicit model (DDIM), which extended DDPM to a non-Markov diffusion process, enabling a deterministic sampling process and the accelerated generation of high-quality samples. In addition to DDPM and DDIM, diffusion models have also been studied with a score-matching probabilistic model using Langevin dynamics \citep{song2019generative, song2020score}. Diffusion models evolved to encompass guided diffusion models, which are designed to generate conditional distributions. When the conditioning input is text and the output is image, these models transform into text-to-image diffusion models \citep{rombach2022high,ramesh2022hierarchical,ho2022cascaded, saharia2022image}. They bridge the gap between textual descriptions and image synthesis, offering exciting possibilities for content generation. A significant advancement in text-to-image generation is the introduction of Stable Diffusion (SD) \citep{rombach2022high}. SD has expanded the potential of diffusion models by integrating latent variables into the generation process. This innovation in latent diffusion models enables the exploration of latent spaces and improves the diversity of generated content. Despite the introduction of latent spaces, generating images with desired content from text prompts remains a significant challenge \citep{gal2022image,ruiz2023dreambooth}. This is due to the difficulty in learning the semantic properties of text prompts with limited high-quality data.

\noindent\textbf{Fine-Tuning Diffusion Models.}
Efforts to improve diffusion models have focused on aligning them more closely with human preferences. \citet{rombach2022high} fine-tuned a pre-trained model using the COCO dataset \citep{caesar2018coco}, demonstrating superior performance compared to a generative model directly trained on the same dataset. \citet{podell2023sdxl} expanded the model size of Stable Diffusion (SD) to create the SDXL model, which was fine-tuned on a high-quality but private dataset, leading to a significant improvement in the aesthetics of the generated images. \citet{dai2023emu} further demonstrated the effectiveness of fine-tuning and highlighted the importance of the supervised fine-tuning (SFT) dataset. In addition to using datasets with high-quality images, \citet{betker2023improving,segalis2023picture} found that SFT on a data set with high text fidelity can also improve the performance of the diffusion model. The aforementioned methods only requires a high-quality SFT dataset. Recently, preference datasets have been studied in finetuing diffusion models \citep{lee2023aligning}. Concurrently, DDPO \citep{black2023training} and DPOK \citep{fan2023dpok} proposed to use the preference dataset to train a reward model and then fine-tune diffusion models using reinforcement learning. Drawing inspiration from the recent Direct Preference Optimization (DPO) \citep{rafailov2023direct}, Diffusion-DPO \citep{wallace2023diffusion} and D3PO \citep{yang2023using} used the implicit reward to fine-tune diffusion models directly on the preference dataset. Furthermore, when a differentiable reward model is available, \citet{clark2023directly, prabhudesai2023aligning} applied reward backpropagation for fine-tuning diffusion models. Our SPIN-Diffusion is most related to the SFT method, as it only assumes access to high-quality image-text pairs. However, the high-quality image-text dataset can be obtained from various sources, including selecting the winner from a preference dataset or identifying high-reward image-text pairs through a reward model.

\section{Problem Setting and Preliminaries}\label{section:problemsetting}

In this section, we introduce basic settings for text-to-image generation by diffusion models and the self-play fine-tuning (SPIN) method. 

\subsection{Text-to-Image Diffusion Model} 
Denoising diffusion implicit models (DDIM) \citep{song2020denoising} is a generalized framework of denoising diffusion probabilistic models (DDPM) \citep{sohl2015deep, ho2020denoising}. DDIM enables the fast generation of high-quality samples and has been widely used in text-to-image diffusion models such as Stable Diffusion \citep{rombach2022high}. We formulate our method building upon DDIM, which makes it more general.

\noindent\textbf{Forwrd Process.} 
Following \citet{saharia2022image}, the problem of text-to-image generation can be formulated as conditional diffusion models. We use $\xb_0 \in \RR^{d}$ to denote the value of image pixels where $d$ is the dimension and use $\cbb$ to denote the text prompt. Given a prompt $\cbb$, image $\xb_0 $ is drawn from a target data distribution $p_{\mathrm{data}}(\cdot|\cbb)$. The diffusion process is characterized by the following dynamic parameterized by a positive decreasing sequence $\{\alpha_{t}\}_{t=1}^{T}$ with $\alpha_0 = 1$,
\begin{align}
q(\xb_{1: T} |\boldsymbol{x}_0):= q(\xb_T | \xb_0) \prod_{t=2}^T q(\xb_{t-1} | \xb_t, \xb_0),    \label{eq:forward}   
\end{align}
where $q(\xb_{t-1} | \xb_t, \xb_0)$ represents a Gaussian distribution $\cN(\bmu_{t}, \sigma_t^{2}\Ib)$. Here, $\bmu_t$ is the mean of Gaussian defined as
\begin{align*}
\bmu_t:= \sqrt{\alpha_{t-1}} \xb_0  +\sqrt{1-\alpha_{t-1}-\sigma_t^2} \cdot \frac{\xb_t-\sqrt{\alpha_t} \xb_0}{\sqrt{1-\alpha_t}}. 
\end{align*}
It can be derived from \eqref{eq:forward} that $q(\xb_t | \xb_0)= \cN(\sqrt{\alpha_t} \xb_0, (1-\alpha_t) \Ib)$  for all $t$ \citep{song2020denoising}. As a generalized diffusion process of DDPM, \eqref{eq:forward} reduces to DDPM \citep{ho2020denoising} with a special choice of $\sigma_t = \sqrt{(1-\alpha_{t-1})/(1-\alpha_t)}\sqrt{(1-\alpha_t /\alpha_{t-1})}$. 

\noindent\textbf{Generative Process.}
Given the sequence of $\{\alpha_{t}\}_{t=1}^{T}$ and $\{\sigma_t\}_{t=1}^{T}$, examples from the generative model follows 
\begin{align}
p_{\btheta}(\xb_{0:T} | \cbb) = \prod_{t=1}^{T} p_{\btheta}(\xb_{t-1} | \xb_t, \cbb) \cdot p_{\btheta}(\xb_{T} | \cbb), \qquad
p_{\btheta}(\xb_{t-1} | \xb_t, \cbb) = \mathcal{N}\big( \bmu_{\btheta}(\xb_t,\cbb,t), \sigma_{t}^{2} \Ib\big).    \label{eq:backward}
\end{align}
Here $\btheta$ belongs to the parameter space $\bTheta$ and $\bmu_{\btheta}(\xb_t,\cbb,t)$ is the estimator of mean $\bmu_t$ that can be reparameterized \citep{ho2020denoising, song2020denoising} as the combination of $\xb_t$ and a neural network  $\bepsilon_{\btheta}(\xb_t, \cbb, t)$ named score function.
Please see Appendix~\ref{app:detail} for more details.

\noindent\textbf{Training Objective.} The 
score function $\bepsilon_{\btheta}(\xb_t, \cbb, t)$ is trained by minimizing the evidence lower
bound (ELBO) associated with the diffusion models in \eqref{eq:forward} and \eqref{eq:backward}, which is equivalent to minimizing the following denoising score matching objective function $L_{\mathrm{DSM}}$:
\begin{align}
L_{\mathrm{DSM}}(\btheta) = \mathbb{E}\big[\gamma_t\big\| \bepsilon_{\btheta}(\xb_t, \cbb, t) - \bepsilon_t\big\|_2^2\big], \label{eq:DMloss}
\end{align}
where $\xb_t = \sqrt{\alpha_t}\xb_0 + \sqrt{1-\alpha_t}\bepsilon_t$ and the expectation is computed over the distribution $\cbb\sim q(\cdot), \xb_0 \sim q_{\mathrm{data}}(\cdot|\cbb), \bepsilon_t\sim \cN(0, \Ib)$, $t \sim \mathrm{Uniform}\{1,\ldots,T\}$. In addition, $\{\gamma_t\}_{t=1}^{T}$ are pre-specified weights that depends on the sequences $\{\alpha_{t}\}_{t=1}^{T}$ and $\{\sigma_t\}_{t=1}^{T}$.

\subsection{Self-Play Fine-Tuning}
Self-Play mechanism, originating from TD-Gammon \citep{tesauro1995temporal}, has achieved great sucesses in various fields, particularly in strategic games \citep{silver2017mastering, silver2017masteringchess}. Central to Self-Play is the idea of progressively improving a model by competing against its previous iteration. This approach has recently been adapted to fine-tuning Large Language Models (LLMs)~\citep{chen2024self}, called self-play fine-tuning (SPIN).
Considering an LLM where $\cbb$ is the input prompt and $\xb_0$ is the response, the goal of SPIN is to fine-tune an LLM agent, denoted by $p_{\btheta}(\cdot | \cbb)$, based on an SFT dataset.
~\citet{chen2024self} assumed access to a main player and an opponent player at each iteration and takes the following steps iteratively:
\begin{enumerate}[leftmargin=*]
\item The main player maximizes
the expected value gap between the target data distribution $p_{\mathrm{data}}$ and the opponent player’s distribution $p_{\btheta_k}$:
\item The opponent player generates responses that are indistinguishable from $p_{\mathrm{data}}$ by the main player.
\end{enumerate}
Instead of alternating optimization, SPIN directly utilizes a closed-form solution of the opponent player, which results in the opponent player at iteration $k+1$ to copy parameters $\btheta_{k+1}$, and forming an end-to-end training objective:
 \begin{align}
L_{\method} = \EE\bigg[\ell\bigg(\lambda \log \frac{p_{\btheta}(\xb_{0} | \cbb)}{p_{\btheta_k}(\xb_{0} | \cbb)}-\lambda \log \frac{p_{\btheta}(\xb_{0}' | \cbb)}{p_{\btheta_k}(\xb_{0}' | \cbb)}\bigg)\bigg]. \label{eq:spin}
\end{align}
Here the expectation is taken over the distribution $\cbb \sim q(\cbb), \xb \sim p_{\mathrm{data}}(\xb|\cbb), \xb' \sim p_{\btheta_k}(\xb'|\cbb)$, $\ell(\cdot)$ is a loss function that is both monotonically decreasing and convex, and $\lambda>0$ is a hyperparameter. 
Notably,~\eqref{eq:spin} only requires the knowledge of demonstration/SFT data, i.e., prompt-response pairs. 


\section{Method}\label{sec:method}
In this section, we are going to present a method for fine-tuning diffusion models with self-play mechanisam. 

Consider a setting where we are training on a high-quality dataset containing image-text pairs $(\cbb, \xb_0) \sim p_{\mathrm{data}}(\xb_0|\cbb)q(\cbb)$ where $\cbb$ is the text prompt and $\xb_0$ is the image. Our goal is to fine-tune a pretrained diffusion model, denoted by $p_{\btheta}$, to align with the distribution $p_{\mathrm{data}}(\xb_0|\cbb)$. Instead of directly minimizing the denoising score matching objective function $L_{\mathrm{DSM}}$ in \eqref{eq:DMloss}, we adapt $\method$ to diffusion models. However, applying SPIN to fine-tuning diffusion models presents unique challenges.  Specifically, the objective of $\method$ \eqref{eq:spin} necessitates access to the marginal probability $p_{\btheta}(\xb_0|\cbb)$. While obtaining $p_{\btheta}(\xb_0|\cbb)$ is straightforward in LLMs, this is not the case with diffusion models. Given the parameterization of the diffusion model as $p_{\btheta}(\xb_{0:T}|\cbb)$, computing the marginal probability $p_{\btheta}(\xb_0|\cbb)$ requires integration over all potential trajectories $\int_{\xb_{1:T}}p_{\btheta}(\xb_{0:T}|\cbb)d\xb_{1:T}$, which is computationally intractable. 

In the following, we propose a novel SPIN-Diffusion method with a decomposed objective function that only requires the estimation of score function $\bepsilon_{\btheta}$. 
This is achieved by employing the DDIM formulation discussed in Section \ref{section:problemsetting}.
The key technique is self-play mechanism with a focus on the joint distributions of the entire diffusion process, i.e., $p_{\mathrm{data}}(\xb_{0:T}|\cbb) = q(\xb_{1:T}|\xb_0) p_{\mathrm{data}}(\xb_0|\cbb)$ and $p_{\btheta}(\xb_{0:T}|\cbb)$, instead of marginal distributions.

\begin{algorithm}[t]
\caption{Self-Play Diffusion (SPIN-Diffusion)}\label{alg:Improving}
\begin{algorithmic}
\STATE \textbf{Input:} $\{(\xb_0, \cbb)\}_{i\in [N]}$: SFT Dataset, $p_{\btheta_0}$: Diffusion Model with parameter $\btheta_0$, $K$: Number of iterations. 
\FOR{$k= 0,\ldots, K-1$}
\FOR{$i = 1, \ldots N$}
\STATE Generate real diffusion trajectories $\xb_{1:T} \sim q(\xb_{1:T}|\xb_0)$.
\STATE Generate synthetic diffusion trajectories $\xb'_{0:T} \sim p_{\btheta_k}(\cdot|\cbb)$.
\ENDFOR 
\STATE Update $\btheta_{k+1} = \argmin_{\btheta \in \bTheta} \hat{L}_{\method}(\btheta, \btheta_k)$, which is the empirical version of \eqref{eq:loss_score} or \eqref{eq:newloss} .
\ENDFOR
\STATE \textbf{Output:} $\btheta_K$.
\end{algorithmic}
\end{algorithm}

\subsection{Differentiating Diffusion Processes}
In iteration $k+1$, we focus on training a function $f_{k+1}$ to differentiate between the diffusion trajectory $\xb_{0:T}$ generated by the diffusion model parameterized by $p_{\btheta}(\xb_{0:T}|\cbb)$, and the diffusion process $p_{\mathrm{data}}(\xb_{0:T}|\cbb)$ from the data. Specifically, the training of $f_{k+1}$ involves minimizing a generalized Integral Probability Metric (IPM) \citep{muller1997integral}:
\begin{align}
f_{k+1} = \argmin_{f \in \cF_{k}}\EE\big[ \ell\big(f(\cbb, \xb_{0:T}) - f(\cbb, \xb'_{0:T})\big) 
\big].   \label{eq:f-star}
\end{align}
Here, the expectation is taken over the distributions $\cbb\sim q(\cdot), \xb_{0:T}\sim p_{\mathrm{data}}(\cdot|\cbb)$, and $ \xb'_{0:T}\sim p_{\btheta_k}(\cdot|\cbb)$. $\cF_{k}$ denotes the class of functions under consideration and $\ell(\cdot)$ is a monotonically decreasing and convex function that helps stabilize training.  The value of $f$ reflects the degree of belief that the diffusion process $\xb_{0:T}$ given context $\cbb$ originates from the target diffusion process 
$p_{\mathrm{data}}(\xb_{0:T}|\cbb)$
rather than the diffusion model $p_{\btheta}(\xb_{0:T}|\cbb)$. We name $f$ the test function.

\subsection{Deceiving the Test Function} 
The opponent player wants to maximize the expected value $\EE_{\cbb \sim q(\cdot), \xb_{0:T}\sim p(\cdot|\cbb)} [f_{k+1}(\cbb, \xb)]$.
In addition, to prevent excessive deviation of $p_{\btheta_{k+1}}$ from $p_{\btheta_{k}}$ and stabilize the self-play fine-tuning, we incorporate a Kullback-Leibler (KL) regularization term. Putting these together gives rise to the following optimization problem:
\begin{align}
\argmax_ {p}&\EE_{\cbb \sim q(\cdot), \xb_{0:T}\sim p(\cdot|\cbb)} [f_{k+1}(\cbb,\xb_{0:T})] - \lambda \EE_{\cbb\sim q(\cdot)}\mathrm{KL}\big(p(\cdot|\cbb)||p_{\btheta_k}(\cdot|\cbb)\big), \label{eq: update}
\end{align} 
where $\lambda>0$ is the regularization parameter. Notably, \eqref{eq: update} has a closed-form solution $\hat{p}(\cdot|\cbb)$: 
\begin{align}
\hat{p}(\xb_{0:T}|\cbb) \propto p_{\btheta_k}(\xb_{0:T}|\cbb) \exp\big(\lambda^{-1}f_{k+1}(\cbb, \xb_{0:T})\big). \label{eq:closed form solution}  
\end{align}
To ensure that $\hat{p}$ lies in the diffusion process space $\{p_{\btheta}(\cdot|\cbb)|\btheta \in \bTheta\}$, we utilize the following test function class \citep{chen2024self}:
\begin{align}
\cF_{k} = \bigg\{\lambda\cdot \log \frac{p_{\btheta}(\xb_{1:T} | \cbb)}{p_{\mathrm{\btheta_k}}(\xb_{1:T} | \cbb)}\bigg|\btheta \in \bTheta\bigg\}.  \label{eq:function class0} 
\end{align}
Given the choice of $\cF_k$ in \eqref{eq:function class0}, optimizing~\eqref{eq:f-star} gives $f_{k+1}$ parameterized by $\btheta_{k+1}$ in the following form:
\begin{align}
f_{k+1}(\cbb, \xb_{0:T}) = \lambda\cdot \log \frac{p_{\btheta_{k+1}}(\xb_{0:T} | \cbb)}{p_{\mathrm{\btheta_k}}(\xb_{0:T} | \cbb)}.    \label{eq:t+1}
\end{align}
Substituting \eqref{eq:t+1} into \eqref{eq:closed form solution} yields $\hat{p}(\xb_{0:T}|\cbb) = p_{\btheta_{k+1}}(\xb_{0:T}|\cbb)$. 
In other words, $\btheta_{k+1}$ learned from \eqref{eq:f-star} is exactly the diffusion parameter for the ideal choice of opponent. 

\subsection{Decomposed Training Objective} 
The above two steps provide a training scheme depending on the full trajectory of $\xb_{0:T}$. Specifically, substituting~\eqref{eq:function class0} into~\eqref{eq:f-star} yields the update rule $\btheta_{k+1} = \argmin_{\btheta \in \bTheta}L_{\method}(\btheta, \btheta_k)$, where $L_{\method}$ is defined as:
\begin{align}
L_{\method} = \EE\bigg[\ell\bigg(\lambda \log \frac{p_{\btheta}(\xb_{0:T} | \cbb)}{p_{\btheta_k}(\xb_{0:T} | \cbb)}-\lambda \log \frac{p_{\btheta}(\xb_{0:T}' | \cbb)}{p_{\btheta_k}(\xb_{0:T}' | \cbb)}\bigg)\bigg]. \label{eq:loss}  
\end{align}
Here the expectation is taken over the distributions $\cbb \sim q(\cdot), \xb_{0:T} \sim p_{\mathrm{data}}(\cdot|\cbb), \xb'_{0:T} \sim p_{\btheta_k}(\cdot|\cbb)$.
To formulate a computationally feasible objective, we decompose $\log p_{\btheta}(\xb_{0:T} | \cbb)$ using the backward process of diffusion models. Substituting \eqref{eq:backward} into~\eqref{eq:loss}, we have that 
\begin{align}
\log p_{\btheta}(\xb_{0:T} | \cbb) 
&= \log\bigg(\prod_{t=1}^{T} p_{\btheta}(\xb_{t-1} | \xb_t, \cbb) \cdot p_{\btheta}(\xb_{T} | \cbb)\bigg) \notag\\
&= \log p_{\btheta}(\xb_{T} | \cbb) + \sum_{t=1}^{T}\log\big(p_{\btheta}(\xb_{t-1} | \xb_t, \cbb) \big) \notag \\
&= \text{Constant} - \sum_{t=1}^{T}\frac{1}{2\sigma_t^{2}}\big\|\xb_{t-1} - \bmu_{\btheta}(\xb_t,\cbb,t)\big\|_{2}^{2}. \label{eq:p_theta}
\end{align}
where the last equality holds since $p_{\btheta}(\xb_{t-1} | \xb_t, \cbb)$ is a Gaussian distribution $\mathcal{N}\big( \bmu_{\btheta}(\xb_t,\cbb,t), \sigma_{t}^{2} \Ib\big)$ according to \eqref{eq:backward}, and $ p_{\btheta}(\xb_{T} | \cbb)$ is approximately a Gaussian independent of $\btheta$. By substituting \eqref{eq:p_theta} into \eqref{eq:loss} and introducing a reparameterization $\sigma_t^{2} = \lambda T/(2\beta_t)$, where $\beta_{t}$ is a fixed positive value, we obtain
\begin{align}
L_{\method}(\btheta, \btheta_k)&=\EE\bigg[\ell\bigg(- \sum_{t=1}^{T}\frac{\beta_t}{T}\big[\big\|\xb_{t-1} - \bmu_{\btheta}(\xb_t,\cbb,t)\big\|_{2}^{2}  - \big\|\xb_{t-1} - \bmu_{\btheta_k}(\xb_t,\cbb,t)\big\|_{2}^{2}\notag\\
&\qquad- \big\|\xb'_{t-1}- \bmu_{\btheta}(\xb'_t,\cbb,t)\big\|_{2}^{2}
+\big\|\xb'_{t-1} - \bmu_{\btheta_k}(\xb'_t,\cbb,t)\big\|_{2}^{2}\big]\bigg)\bigg]. \label{eq:loss_score}  
\end{align}
Here the expectation is taken over the distributions $\cbb \sim q(\cdot), \xb_{0:T} \sim p_{\mathrm{data}}(\cdot|\cbb), \xb'_{0:T} \sim p_{\btheta_k}(\cdot|\cbb)$.
The detailed algorithm is presented in Algorithm~\ref{alg:Improving}. \eqref{eq:loss_score} naturally provides an objective function for DDIM with $\sigma_{t} > 0$, where $\sigma_t$ controls the determinism of the reverse process \eqref{eq:backward}. \eqref{eq:loss_score} remains valid for deterministic generation processes as $\sigma_{t} \rightarrow 0$.




\subsection{Approximate Training Objective}
While \eqref{eq:loss_score} is the exact ELBO, optimizing it requires storing all intermediate images during the reverse sampling, which is not memory-efficient. To address this limitation, we propose an approximate objective function.  By applying Jensen’s inequality and the convexity of the loss function $\ell$, we can give an upper bound of \eqref{eq:loss_score} and thus move the average over $t$ outside the loss function $\ell$: 
\begin{align}
L_{\method}^{\mathrm{approx}}(\btheta, \btheta_k)&=\EE\bigg[\ell\Big(- \beta_t\Big[\big\|\xb_{t-1} - \bmu_{\btheta}(\xb_t,\cbb,t)\big\|_{2}^{2}  - \big\|\xb_{t-1} - \bmu_{\btheta_k}(\xb_t,\cbb,t)\big\|_{2}^{2} \notag\\
&\qquad
- \big\|\xb'_{t-1}- \bmu_{\btheta}(\xb'_t,\cbb,t)\big\|_{2}^{2}+\big\|\xb'_{t-1} - \bmu_{\btheta_k}(\xb'_t,\cbb,t)\big\|_{2}^{2}\Big]\Big)\bigg], \label{eq:newloss}    
\end{align}
where the expectation is taken over the distributions $\cbb \sim q(\cbb), (\xb_{t-1}, \xb_{t}) \sim p_{\mathrm{data}}(\xb_{t-1}, \xb_{t}|\cbb), (\xb'_{t-1}, \xb'_{t}) \sim  p_{\btheta_{k}}(\xb'_{t-1}, \xb'_{t}|\cbb)$, $t \sim \mathrm{Uniform}\{1,\ldots,T\}$. 

The following lemma shows that $L_{\method}^{\mathrm{approx}}$ is an upper bound of $L_{\method}$. 
\begin{lemma}\label{lm:uppderbound}
Fix $\btheta_k \in \bTheta$ which serves as the starting point of Algorithm~\ref{alg:Improving} for iteration $k+1$. It holds that $L_{\method}(\btheta, \btheta_k) \leq L_{\method}^{\mathrm{approx}}(\btheta, \btheta_k)$ for all $\btheta \in \bTheta$.
\end{lemma}

$L_{\method}^{\mathrm{approx}}$ eliminates the need to store all intermediate steps, as it only involves two consecutive sampling steps $t-1$ and $t$. Since the reverse process $p_{\btheta}(\xb'_{1:T}|\xb'_0, \cbb)$ approximates the forward process $q(\xb'_{1:T}|\xb'_0)$, we use the per step forward process 
$q(\xb'_{t-1}, \xb'_{t}|\xb'_{0})$ to approximate $p_{\btheta_{k}}(\xb'_{t-1}, \xb'_{t}|\xb'_0,\cbb)$. We can further approximate $p_{\btheta_{k}}(\xb'_{t-1}, \xb'_{t}|\cbb) = \int p_{\btheta_{k}}(\xb'_{t-1}, \xb'_{t}|\xb'_0,\cbb)p_{\btheta_{k}}(\xb'_0|\cbb) d\xb'_0$ with $\int q(\xb'_{t-1}, \xb'_{t}|\xb'_{0})p_{\btheta_{k}}(\xb'_0|\cbb)d\xb'_0$. 
Substituting the corresponding terms in \eqref{eq:newloss} with the above approximation allows us to only compute the expectation of \eqref{eq:newloss} over the distribution $\cbb \sim q(\cbb)$,  $(\xb_{t-1}, \xb_{t}) \sim p_{\mathrm{data}}(\xb_{t-1}, \xb_{t}|\cbb)$, $(\xb_{t-1}', \xb'_{t}) \sim \int p_{\btheta_{k}}(\xb'_{0}|\cbb)q(\xb'_{t-1}, \xb'_{t}|\xb'_{0})d\xb'_0$, $t \sim \mathrm{Uniform}\{1,\ldots,T\}$. Furthermore, by incorporating the reparameterization of $\bmu_{\btheta}$ into \eqref{eq:loss_score} and \eqref{eq:newloss}, we can express \eqref{eq:loss_score} and \eqref{eq:newloss} in terms of $\bepsilon_{\btheta}(\xb_{t}, \cbb, t)$. Detailed derivations of \eqref{eq:loss_score} and \eqref{eq:newloss} are provided in Appendix \ref{app:detail}.

\section{Main Theory}\label{sec:thm}
In this section, we provide a theoretical analysis of Algorithm~\ref{alg:Improving}. Section~\ref{sec:method} introduces two distinct objective functions, as defined in \eqref{eq:loss_score} and \eqref{eq:newloss}, both of which use the loss function $\ell$. 
Since \eqref{eq:loss_score} is an exact objective function, its analysis closely follows the framework established by \citet{chen2024self}. Consequently, we instead focus on the approximate objective function $L_{\method}^{\mathrm{approx}}$ defined in~\eqref{eq:newloss}, which is more efficient to optimize and is the algorithm we use in our experiments.  However, its behavior is more difficult to analyze.  We begin with a formal assumption regarding the loss function $\ell$ as follows.
\begin{assumption} \label{assm:1}
The loss function $\ell(t): \RR \rightarrow \RR$ is monotonically decreasing, i.e., $\forall t, \ell'(t) \leq 0$ and satisfies $\ell'(0) < 0$. In addition, $\ell(t)$ is a convex function.
\end{assumption}
Assumption~\ref{assm:1} can be satisfied by various commonly used loss functions in machine learning. This includes the correlation loss $\ell(t) = 1 - t$, the hinge loss $\ell(t) = \max(0,1 - t)$, and the logistic loss $\ell(t) = \log(1 + \exp(-t))$.

To understand the behavior of SPIN-Diffusion, let us first analyze the gradient of the objective function~\eqref{eq:newloss}, 
\begin{align}
\nabla L_{\method}^{\mathrm{approx}} =\EE\Big[\underbrace{(-\beta_t\ell'_{t})}_{\mathrm{Reweighting}}\cdot\big( \underbrace{\nabla_{\btheta}\big\|\xb_{t-1} - \bmu_{\btheta}(\xb_t,\cbb,t)\big\|_{2}^{2}}_{\mathrm{Matching}} - \underbrace{\nabla_{\btheta}\big\|\xb'_{t-1}- \bmu_{\btheta}(\xb'_t,\cbb,t)\big\|_{2}^{2}}_{\mathrm{Pushing}}\big)\Big],  \label{eq:gradient}
\end{align}
where the expectation is taken over the distributions $\cbb \sim q(\cbb), (\xb_{t-1}, \xb_{t}) \sim p_{\mathrm{data}}(\xb_{t-1}, \xb_{t}|\cbb), (\xb'_{t-1}, \xb'_{t}) \sim  p_{\btheta_{k}}(\xb'_{t-1}, \xb'_{t}|\cbb)$.~\eqref{eq:gradient} can be divided into three parts:
\begin{itemize}[leftmargin=*]
 \item \textbf{Reweighting: } $\ell'(\cdot)$ in the ``Reweighting'' term is negative and increasing  because $\ell()$ is monotonically decreasing and convex according to Assumption~\ref{assm:1}.
 Therefore, $-\beta_{t}\ell'_{t} = -\beta_t\ell'\big(- \beta_t\big[\|\xb_{t-1} - \bmu_{\btheta}(\xb_t,\cbb,t)\|_{2}^{2}- \ldots +\|\xb'_{t-1} - \bmu_{\btheta_k}(\xb'_t,\cbb,t)\|_{2}^{2}\big]\big)$  is always non-negative. Furthermore, $-\beta_{t}\ell'_{t}$ decreases as the argument inside $\ell()$ increases.
 \item \textbf{Matching: } The ``Matching'' term matches $\bmu_{\btheta}(\xb_{t}, \cbb, t)$ to $\xb_{t-1}$ coming from pairs $(\xb_{t-1}, \xb_{t})$, that are sampled from the target distribution.  This increases the likelihood of $(\xb_{t-1}, \xb_{t})\sim p_{\mathrm{data}}(\xb_{t-1}, \xb_{t})$  following the generative process \eqref{eq:backward}. 
 \item \textbf{Pushing: } Contrary to the ``Matching'' term, the ``Pushing'' term pushes $\bmu_{\btheta}(\xb_{t}', \cbb, t)$ away from  $\xb_{t-1}'$ coming from pairs $(\xb'_{t-1}, \xb'_{t})$ drawn from the synthetic distribution $p_{\btheta_{k}}(\xb'_{t-1}, \xb'_{t})$. Therefore, the ``Pushing'' term decreases the likelihood of these samples following the process in the generative process~\eqref{eq:backward}.
\end{itemize}

The ``Matching'' term aligns conceptually with the $L_{\mathrm{DSM}}$ in SFT, as both aim to maximize the likelihood that the target trajectory $\mathbf{x}_{0:T}$ follows the generative process described in \eqref{eq:backward}. The following theorem shows a formal connection, which is pivotal for understanding the optimization dynamics of our method.

\begin{theorem}\label{thm:notstop}
Under Assumption~\ref{assm:1}, if $\btheta_{k}$ is not the global optimum of $L_{\mathrm{DSM}}$ in \eqref{eq:DMloss}, there exists an appropriately chosen $\beta_t$, such that $\btheta_{k}$ is not the global minimum of \eqref{eq:newloss} and thus $\btheta_{k+1}\not=\btheta_k$.
\end{theorem}
Theorem~\ref{thm:notstop} suggests that the optimization process stops only when $\btheta$ reaches global optimality of $L_{\mathrm{DSM}}$. Consequently, the optimal diffusion model $\btheta^{*}$ found by Algorithm~\ref{alg:Improving} cannot be further improved using $L_{\mathrm{DSM}}$. 
This theoretically supports that SFT with \eqref{eq:DMloss} cannot improve over SPIN-Diffusion. It is also worth noting that Theorem~\ref{thm:notstop} does not assert that every global minimum of $L_{\mathrm{DSM}}$ meets the convergence criterion (i.e., $\btheta_{k+1} = \btheta_{k}$), particularly due to the influence of the ``Pushing'' term in \eqref{eq:gradient}. The following theorem provides additional insight into the conditions under which Algorithm~\ref{alg:Improving} converges.

\begin{theorem}\label{thm:stop}
Under Assumption~\ref{assm:1}, if $p_{\btheta_k}(\cdot|\xb) = p_{\mathrm{data}}(\cdot|\xb)$, then $\btheta_k$ is the global minimum of~$\eqref{eq:newloss}$ for any $\lambda\geq 0$. 
\end{theorem}
Theorem~\ref{thm:stop} shows that Algorithm~\ref{alg:Improving} converges when $p_{\btheta}(\cdot|\xb) = p_{\mathrm{data}}(\cdot|\xb)$, 
indicating the efficacy of SPIN-Diffusion in aligning with the target data distribution. In addition, while Theorems~\ref{thm:notstop} and~\ref{thm:stop} are directly applicable to \eqref{eq:newloss}, the analogous conclusion can be drawn for \eqref{eq:loss_score} as well (see Appendix~\ref{sec:proof} for a detailed discussion).

\section{Experiments}~\label{sec: exp}

In this section, we conduct extensive experiments to demonstrate the effectiveness of SPIN-Diffusion. Our results show that SPIN-Diffusion outperforms other baseline fine-tuning methods including SFT and Diffusion-DPO.

\subsection{Experiment Setup}\label{subsec:exp_setup}

\noindent\textbf{Models, Datasets and Baselines.} 
We use the stable diffusion v1.5 (SD-1.5) \citep{rombach2022high} as our base model. While adopting the original network structure, we use its Huggingface pretrained version\footnote{https://huggingface.co/runwayml/stable-diffusion-v1-5}, which is trained on LAION-5B \citep{schuhmann2022laion} dataset, a text-image pair dataset containing approximately 5.85 billion CLIP-filtered image-text pairs. We use the Pick-a-Pic dataset \citep{kirstain2023pick} for fine-tuning. Pick-a-Pic is a dataset with pairs of images generated by Dreamlike\footnote{https://dreamlike.art/} 
(a fine-tuned version of SD-1.5) and SDXL-beta \citep{podell2023sdxl}, where each pair corresponds to a human preference label. 
We also train SD-1.5 with SFT and Diffusion-DPO~\citep{wallace2023diffusion} 
as the baselines. For SFT, we train the model to fit the winner images in the Pick-a-Pic~\citep{kirstain2023pick} trainset.
In addition to the Diffusion-DPO checkpoint provided by \citet{wallace2023diffusion}\footnote{https://huggingface.co/mhdang/dpo-sd1.5-text2image-v1} (denoted by Diffusion-DPO), we also fine-tune an SD-1.5 using Diffusion-DPO and denote it by ``Diffusion-DPO (ours)''.

\noindent\textbf{Evaluation.} 
We use the Pick-a-Pic test set, PartiPrompts \citep{yu2022scaling} and HPSv2 \citep{wu2023human} as our evaluation benchmarks.  All of these datasets are collections of prompts and their size is summarized in Table~\ref{tab:dataset}. Due to space limit, we defer the detailed introduction and results of PartiPrompts and HPSv2 to Appendix~\ref{subsec:ablation}.
Our evaluation rubric contains two dimensions, human preference alignment and visual appeal. For visual appeal assessment, we follow \citet{wallace2023diffusion, lee2024parrot} and use Aesthetic score. For human-preference alignment, we employ reward models including PickScore \citep{kirstain2023pick}, ImageReward \citep{xu2023imagereward} and HPS~\citep{wu2023human}. 
All these reward models are trained according to the Bradley-Terry-Luce \citep{bradley1952rank} model on different human-labeled preference datasets. For each prompt, we generate $5$ images and choose the image with highest average score over those four metrics (best out of $5$). We report the average of HPS, PickScore, ImageReward and Aesthetic scores over all the prompts.
To investigate how the scores align with human preference, we further compare the accuracy of these reward models on a small portion of the Pick-a-Pic training set. It is worth noticing that PickScore is most aligned with human preference according to the experiments conducted by \citet{kirstain2023pick}. The detailed results are shown in Table~\ref{tab:score-align}.

\begin{table}[t!]
    \centering
    \caption{The size of benchmark datasets in our evaluation}
    \begin{tabular}{c | c c c c c c}
    \toprule
        Benchmarks & Pick-a-Pic & PartiPrompts & HPSv2  \\
        \midrule
        \# Prompts & 500 & 1630 & 3200  \\
    \bottomrule
    \end{tabular}%
    \label{tab:dataset}
\end{table}

 \begin{table}[h!]
    \centering
    \caption{The winning rate of the winner image against the loser image in a sample (i.e., 500 text prompts) of the Pick-a-Pic training set in terms of the four metrics. }
    \label{tab:score-align}
    \begin{tabular}{c | c c c c}
    \toprule
        Metrics & PickScore & HPS & Aesthetic & ImageReward \\
        \midrule
        Winning Rate & 74.07 & 61.54 & 51.89 & 62.00 \\
    \bottomrule
    \end{tabular}%
\end{table}


\begin{figure*}[ht!]
\centering   
\subfigure[Aesthetic]{\label{fig:main_aes}\includegraphics[width=0.45\textwidth]{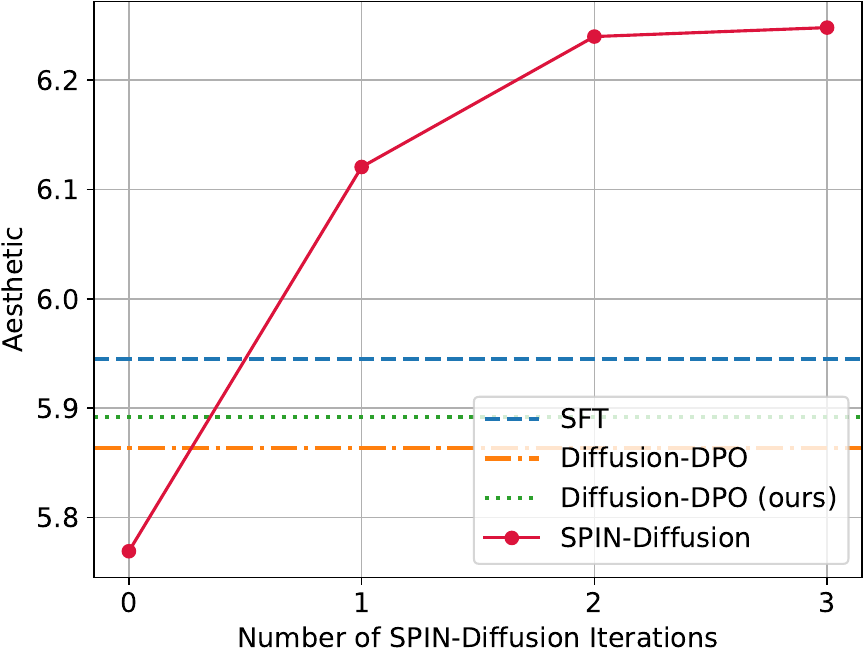}}
\subfigure[PickScore]{\label{fig:main_pic}\includegraphics[width=0.45\textwidth]{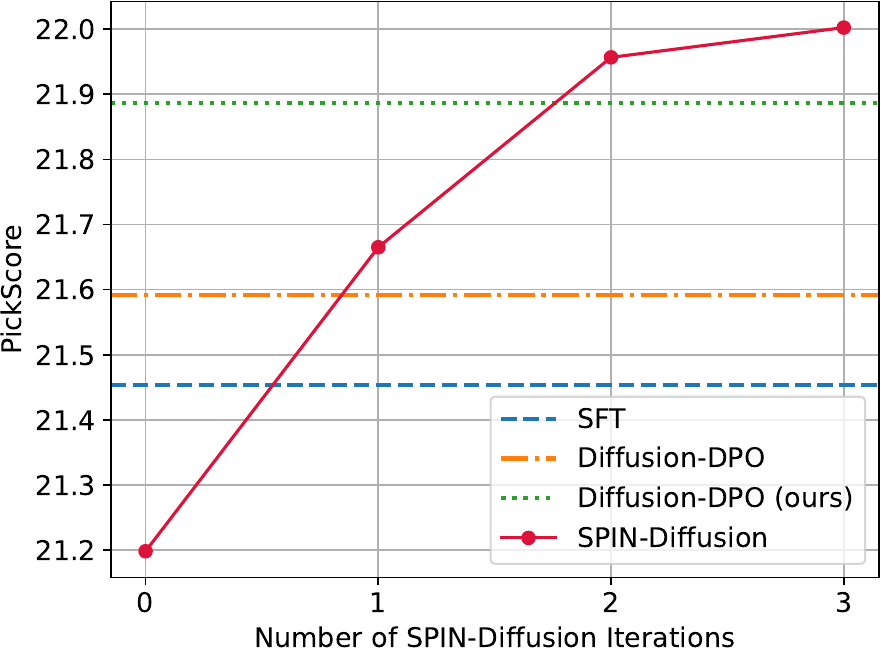}}
\subfigure[HPS]{\label{fig:main_hps}\includegraphics[width=0.45\textwidth]{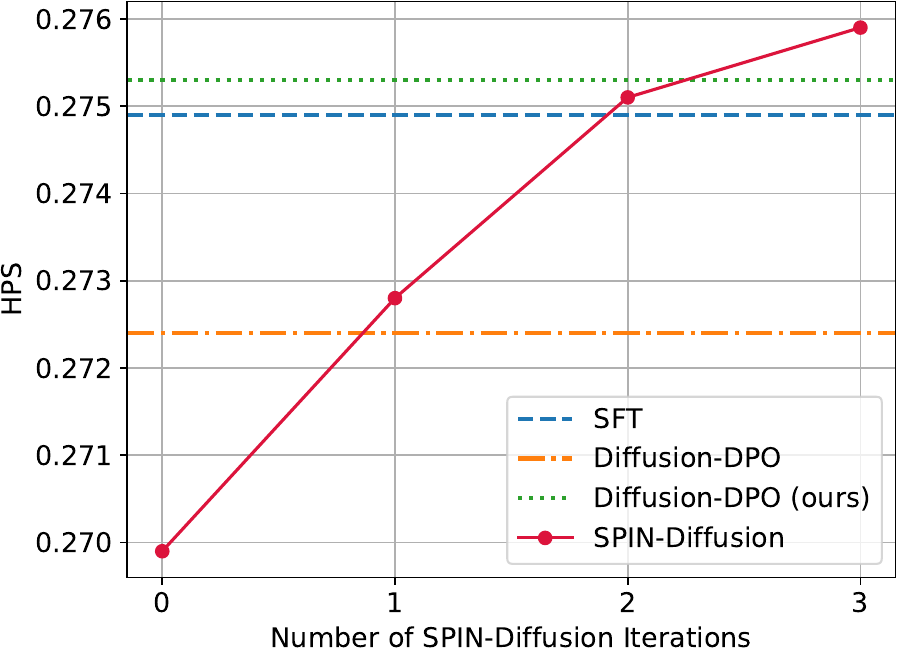}}
\subfigure[ImageReward]{\label{fig:main_hps}\includegraphics[width=0.45\textwidth]{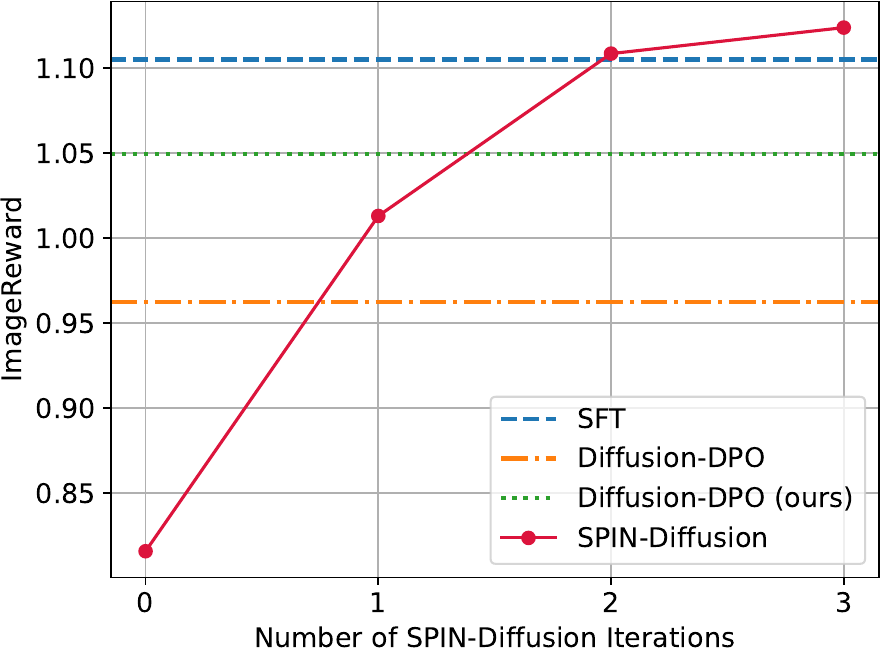}}
\caption{Comparison between SPIN-Diffusion at different iterations with SD-1.5, SFT and Diffusion-DPO. SPIN-Diffusion outperforms SFT at iteration 1, and outperforms all the baselines after iteration~2.}
\label{fig:fig_main}
\end{figure*}

\subsection{Main Results}

In this subsection, we provide empirical evidence demonstrating the superiority of our SPIN-Diffusion model over previous fine-tuning baselines based on the network structure of SD1.5.

\noindent\textbf{Comparison in Terms of Average Score.} The results are presented in Figure~\ref{fig:fig_main} and Table~\ref{tab:results_main_mean}. While all fine-tuning algorithms yield improvements over the SD1.5 baseline, at iteration 1, our SPIN-Diffusion not only exceeds the original DPO checkpoint but also surpasses SFT in both Aesthetic score and PickScore.
\begin{table}[t!]
    \caption{The results on the Pick-a-Pic test set. We report the mean of PickScore, HPS, ImageReward and Aesthetic over the whole test set. We also report the average score over the three evaluation metrics. SPIN-Diffusion outperforms all the baselines in terms of four metrics. For this and following tables, we use blue background to indicate our method, \textbf{bold} numbers to denote the best and \underline{underlined} for the second best.}
    \centering
    \label{tab:results_main_mean}
    \begin{tabular}{l |c c c c c c c c}
    \toprule
        Model & HPS $\uparrow$ & Aesthetic $\uparrow$ & ImageReward $\uparrow$ & PickScore $\uparrow$ & Average $\uparrow$ \\
        \midrule
    SD-1.5 & 0.2699 & 5.7691 & 0.8159 & 21.1983 & 7.0133   \\
    SFT (ours) & 0.2749 & 5.9451 & 1.1051 & 21.4542 & 7.1948 \\
    Diffusion-DPO & 0.2724 & 5.8635 & 0.9625 & 21.5919 & 7.1726 \\
    Diffusion-DPO (ours) & \underline{0.2753} & 5.8918 & 1.0495 & 21.8866 & 7.2758   \\
    \midrule
    \rowcolor{LightCyan} SPIN-Diffusion-Iter1 & 0.2728 & 6.1206 & 1.0131 & 21.6651 & 7.2679   \\
    \rowcolor{LightCyan} SPIN-Diffusion-Iter2 & 0.2751 & \underline{6.2399} & \underline{1.1086} & \underline{21.9567} & \underline{7.3951} \\
    \rowcolor{LightCyan} SPIN-Diffusion-Iter3 & \textbf{0.2759} & \textbf{6.2481} & \textbf{1.1239} & \textbf{22.0024} & \textbf{7.4126} \\
    \bottomrule
    \end{tabular}%
\end{table}
\begin{figure*}[ht!]
    \centering
\includegraphics[width=0.75\textwidth]{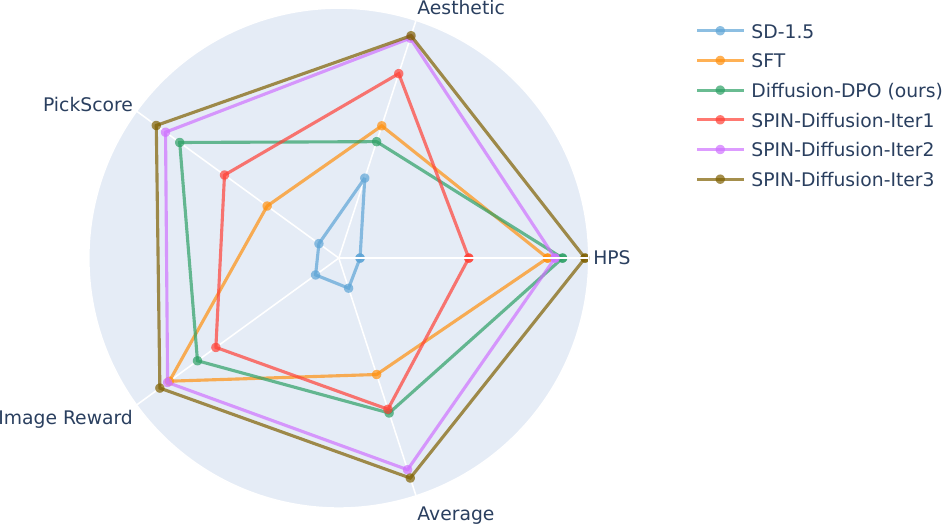}
    \caption{The main result is presented in radar chart. The scores are adjusted to be shown on the same scale. Compared with the baselines, $\method$ achieves higher scores in all the four metrics and the average score by a large margin.}
\label{fig:radar}
\end{figure*}
At iteration 2, the superiority of our model becomes even more pronounced, particularly in terms of Aesthetic score, where it consistently outperforms other fine-tuning methods, indicating a dominant performance in visual quality.
Furthermore, at iteration 3, our model's HPSv2 score surpasses all competing models, highlighting the effectiveness and robustness of the SPIN-Diffusion approach. 
Specifically, on the Pick-a-Pic dataset, while SFT achieves a PickScore of $21.45$, and Diffusion-DPO has a slightly higher score of $21.45$, SPIN-Diffusion achieves $22.00$ at iteration $3$, showing a total improvement of $0.80$ over the original SD1.5 checkpoint. Furthermore, SPIN-Diffusion demonstrates exceptional performance in terms of Aesthetic score, achieving $6.25$ at iteration $3$, which significantly surpasses $5.86$ achieved by Diffusion-DPO and $5.77$ by SD1.5.
The results are also summarized as a radar chart in Figure~\ref{fig:radar}.



\begin{figure*}[ht!]
\centering   
\subfigure[Compared to SD-1.5]{\label{fig:win_sd}\includegraphics[width=0.45\textwidth]{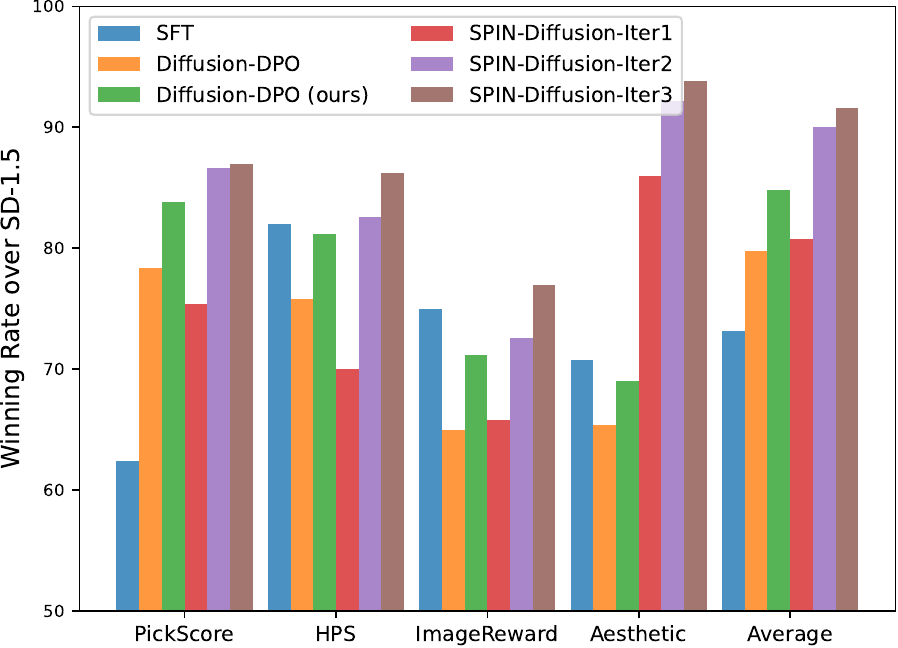}}
\subfigure[Compared to Diffusion-DPO (ours)]{\label{fig:win_dpo}\includegraphics[width=0.45\textwidth]{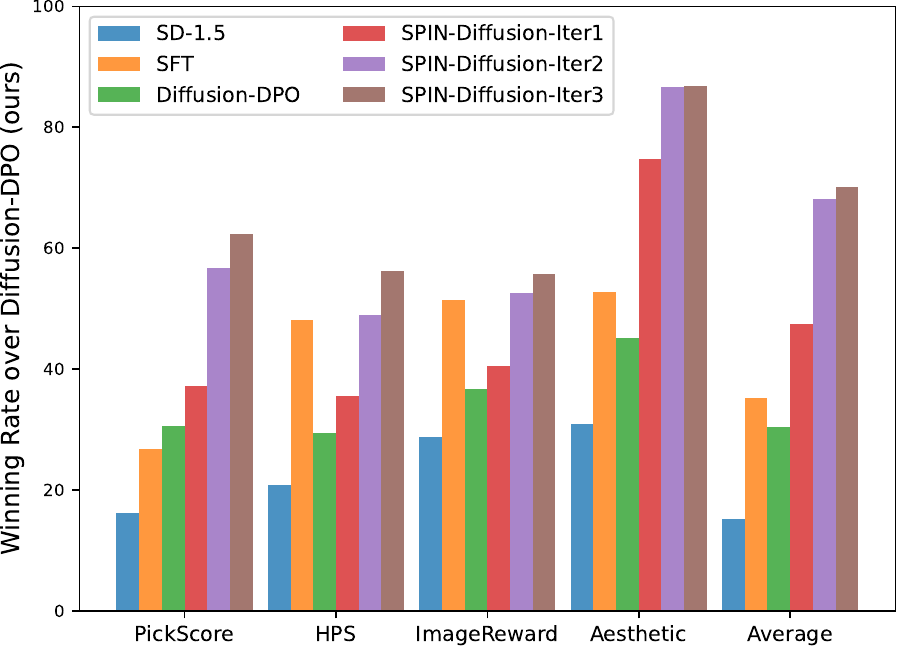}}
\caption{Left: winning rate in percentage of SFT, Diffusion-DPO, Diffusion-DPO (ours) and SPIN-Diffusion over SD1.5 checkpoint. Right:  winning rate in percentage of SFT, Diffusion-DPO, Diffusion-DPO (ours) and SPIN-Diffusion over SD1.5 checkpoint. SPIN-Diffusion shows a much higher winning rate than SFT and Diffusion-DPO tuned models.}
\label{fig:bar_win_rate}
\end{figure*}
\begin{figure*}[ht!]
    \centering
    \includegraphics[width=0.99\textwidth]{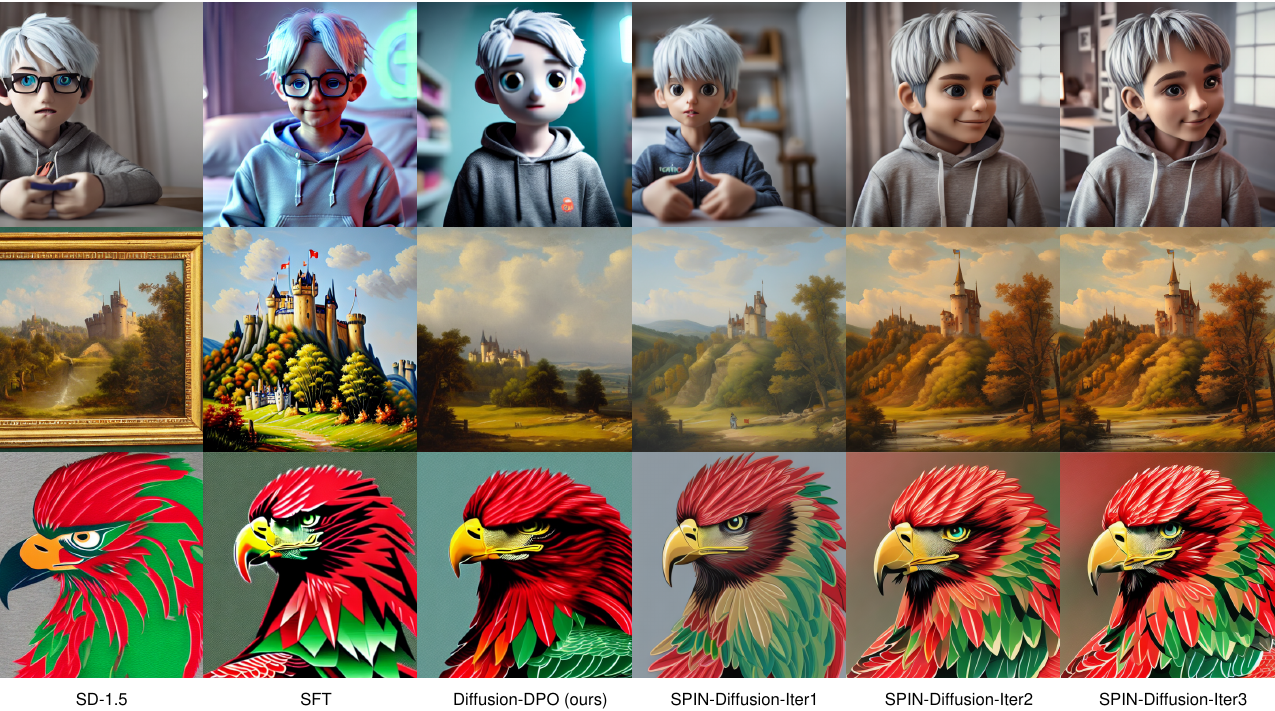}
    \caption{We show the images generated by different models. The prompts are ``\textit{a very cute boy, looking at audience, silver hair, in his room, wearing hoodie, at daytime, ai language model, 3d art, c4d, blender, pop mart, blind box, clay material, pixar trend, animation lighting, depth of field, ultra detailed}'', ``\textit{painting of a castle in the distance}'' and ``\textit{red and green eagle}''. The models are: SD-1.5, SFT, Diffusion-DPO (ours), SPIN-Diffusion-Iter1, SPIN-Diffusion-Iter2, SPIN-Diffusion-Iter3 from left to right. SPIN-Diffusion demonstrates a notable improvement in image quality.}
\label{fig:fig_qualitative}
\end{figure*}
\noindent\textbf{Comparison in Terms of Winning Rate. } We further validate our claim by a comparative analysis of the winning rate for our trained model. The winning rate is defined as the proportion of prompts for which a model's generated images exceed the quality of those produced by another model.
This experiment is conducted on the Pick-a-Pic test set.
We show both the winning rate over SD-1.5, as well as the winning rate over Diffusion-DPO (ours) in Figure~\ref{fig:bar_win_rate}.
The complete results are detailed in Tables~\ref{tab:results_winrate_sd} and~\ref{tab:results_winrate_dpo} in Appendix~\ref{app:add_resu}. We observe that throughout fine-tuning, our SPIN-Diffusion tremendously beats the baselines.
When competing with SD-1.5, SPIN-Diffusion achieves an impressive winning rate of 90.0\% at iteration 2, which further increases to 91.6\% at iteration 3. This winning rate surpasses 73.2\% achieved by SFT and 84.8\% achieved by Diffusion-DPO (ours).
When competing with Diffusion-DPO (ours), at iteration 3, SPIN-Diffusion achieves a winning rate of 56.2\% on HPS, 86.8\% on Aesthetic, 62.4\% on PickScore, 55.8\% on Image Reward, and has an overall winning rate of 70.2\%.


\subsection{Qualitative Analysis}
We illustrate the qualitative performance of our model on three prompts coming from the Pick-a-Pic test dataset. We prompt SD-1.5, SFT, Diffusion-DPO (ours), and SPIN-Diffusion at iteration 1 to 3 and present the generated images in Figure~\ref{fig:fig_qualitative}. Compared to the baseline methods, SPIN-Diffusion demonstrates a notable improvement in image quality, even more apparent than the improvements in scores.
This is especially evident in aspects such as aligning, shading, visual appeal, and the intricacy of details within each image. This qualitative assessment underscores the effectiveness of SPIN-Diffusion in producing images that are not only contextually accurate but also visually superior to those generated by other existing models.

\begin{figure*}[h!]
\centering   
\subfigure[Aesthetic]{\label{fig:curve_aes}\includegraphics[width=0.45\textwidth]{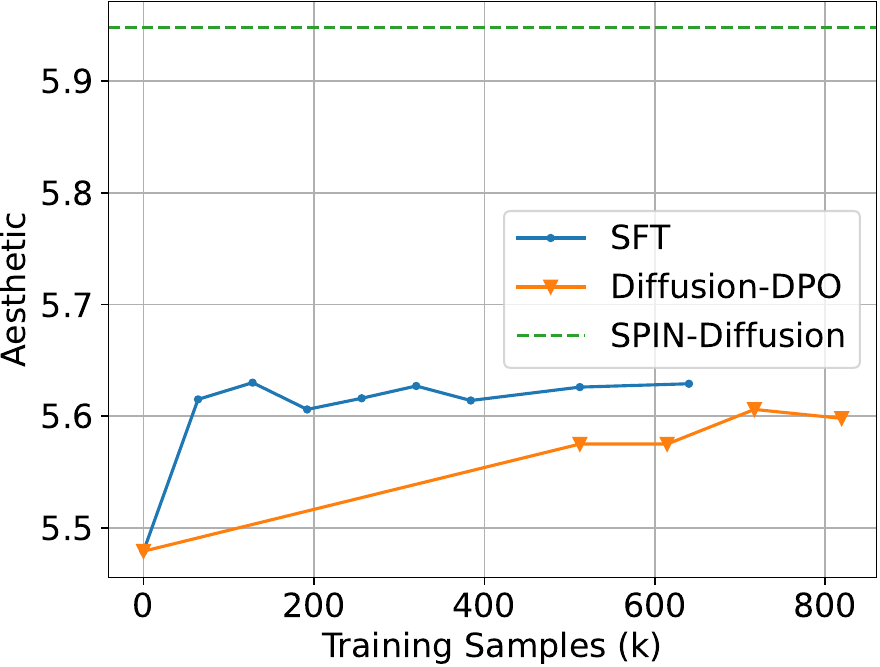}}
\subfigure[PickScore]{\label{fig:curve_pic}\includegraphics[width=0.45\textwidth]{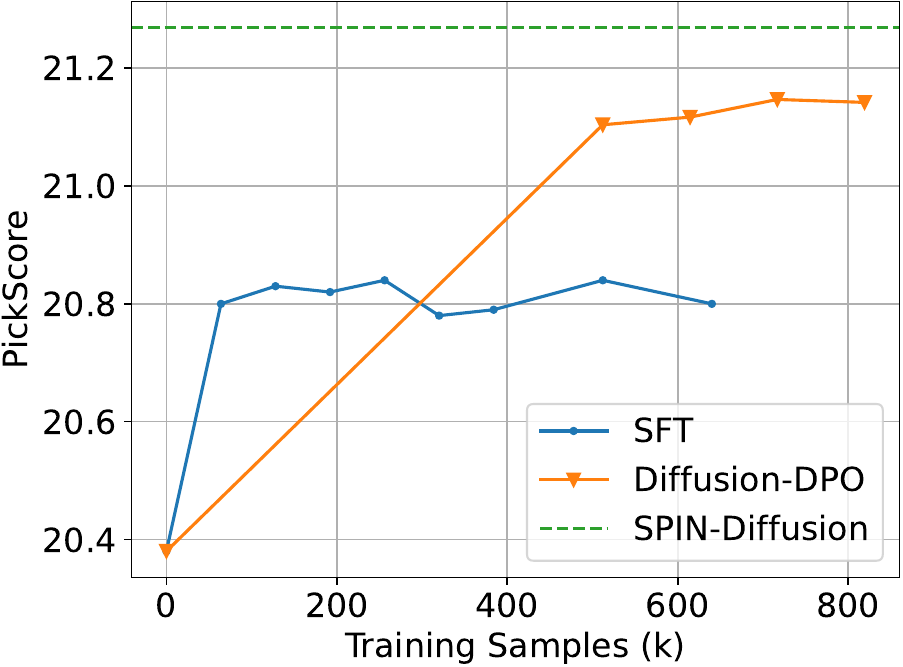}}
\subfigure[HPS]{\label{fig:curve_hps}\includegraphics[width=0.45\textwidth]{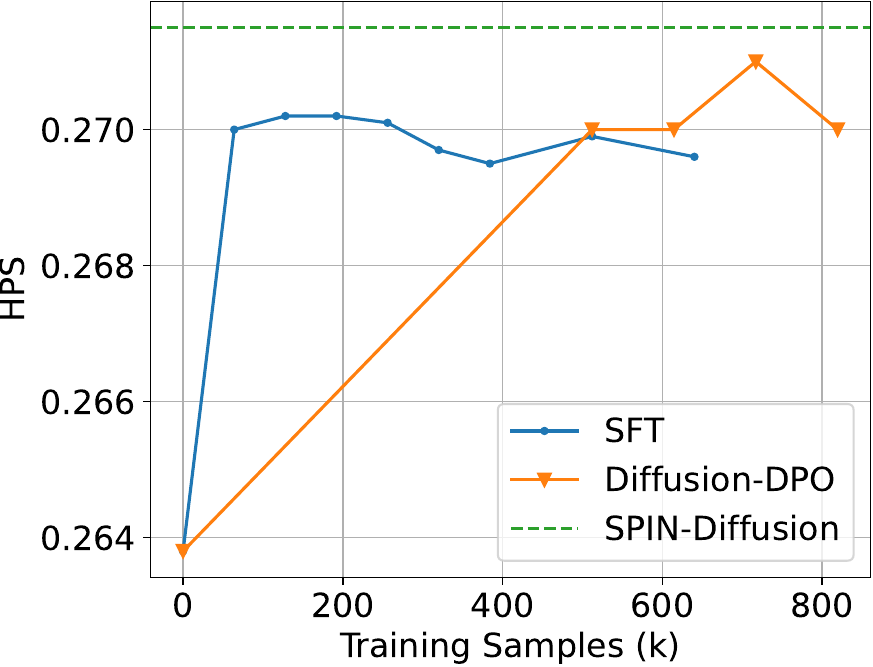}}
\subfigure[Average Score]{\label{fig:curve_avg}\includegraphics[width=0.45\textwidth]{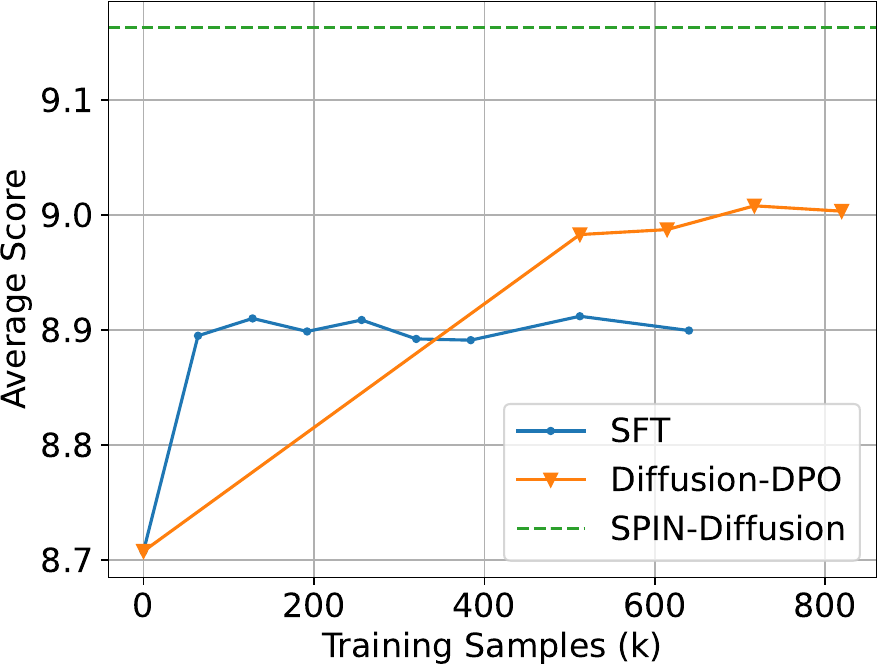}}
\caption{The evaluation results on the Pick-a-Pic validation set of SFT, Diffusion-DPO and SPIN-Diffusion. The x-axis is the number of training data. SFT reaches its limit quickly, while Diffusion-DPO and SPIN-Diffusion continue to improve after training with over 800k data.}
\label{fig:fig_curve}
\end{figure*}

\subsection{Training Dynamics of SFT and DPO} We first study the training dynamic of SPIN-Diffusion in comparison with SFT and Diffusion-DPO,  and we plot the results in Figure~\ref{fig:fig_curve}. We observe that after training with about 50k data, the performance of SFT stop improving and maintains at about 20.8 in PickScore, 0.270 in HPS, 5.6 in Aesthetic and 8.9 in average score. These results is significantly inferior to those achieved by SPIN-Diffusion, which achieves 21.2 in PickScore, 0.272 in HPS, 5.9 in Aesthetic and 9.1 in average score. Compared to Diffusion-DPO, SPIN-Diffusion achieves a superior performance without the loser image.  These results demonstrate that self-play fine-tuning plays a key role in SPIN-Diffusion's performance.


\section{Conclusion}

This paper presents SPIN-Diffusion, an innovative fine-tuning approach tailored for diffusion models, particularly effective in scenarios where only a single image is available per text prompt. By employing a self-play mechanism, SPIN-Diffusion iteratively refines the model's performance, converging towards the target data distribution. Theoretical evidence underpins the superiority of SPIN-Diffusion, demonstrating that traditional supervised fine-tuning cannot surpass its stationary point, achievable at the target data distribution. Empirical evaluations highlight SPIN-Diffusion's remarkable success in text-to-image generation tasks, surpassing the state-of-the-art fine-tuning methods even without the need for additional data. This underscores SPIN-Diffusion's potential to revolutionize the practice of diffusion model fine-tuning, leveraging solely demonstration data to achieve unprecedented performance levels.

\appendix

\section{Additional Details for Experiments}
\subsection{Hyperparameters}
We train the SPIN-Diffusion on 8 NVIDIA A100 GPUs with 80G memory. In training the SPIN-Diffusion, we use the AdamW optimizer with a weight decay factor of $1e-2$. The images are processed at a $512 \times 512$ resolution. The batch size is set to 8 locally, alongside a gradient accumulation of 32. For the learning rate, we use a schedule starting with $200$ warm-up steps, followed by linear decay. We set the learning rate at $2.0e-5$ for the initial two iterations, reducing it to $5.0e-8$ for the third iteration. The coefficient $\beta_t$ is chosen as $2000$ for the first iteration, increasing to $5000$ for the subsequent second and third iterations. Training steps are $50$ for the first iteration, $500$ for the second, and $200$ for the third.
In training the DPO model, we employ the same AdamW optimizer and maintain a batch size of 8 and a gradient accumulation of 32. The learning rate is set to $2.0e-5$, and $\beta_t$ is set to $2000$. The total number of training steps for DPO is $350$. In SFT training, we use 4 NVIDIA A6000 GPUs. We use the AdamW optimizer with a weight decay of $0.01$. The local batch size is set to 32 and the global batch size is set to 512. Our learning rate is 1e-5, with linear warmup for 500 steps with no learning rate decay. We save checkpoints every 500 steps and evaluate the checkpoints on Pick-a-Pic validation. We select the best checkpoint, trained after 2000 steps as our SFT checkpoint.

During generation, we use a guidance scale of $7.5$, and fixed the random seed as $5775709$.

\subsection{Additional Results}\label{app:add_resu}

We present the median scores of baselines and SPIN-Diffusion on Pick-a-Pic testset in Table~\ref{tab:results_median}. The results are consistent to the results in Table~\ref{tab:results_main_mean}. We present the detailed winning rate of baselines and SPIN-Diffusion over SD-1.5 in Table~\ref{tab:results_winrate_sd} and the winning rate over Diffusion-DPO in Table~\ref{tab:results_winrate_dpo}. 

\begin{table*}[ht!]
    \caption{The results of median scores on Pick-a-Pic test set. We report the median of PickScore, HPSv2, ImageReward and Aesthetic over the whole test set. We also report the average score over the four evaluation metric. SPIN-Diffusion outperforms all the baselines regarding HPS, Aesthetic, PickScore and the average score, which agrees with the results of mean scores.}
    \label{tab:results_median}
    \centering
    \begin{tabular}{l |c c c c c c c c}
    \toprule
        Model & HPS $\uparrow$ & Aesthetic $\uparrow$ & ImageReward $\uparrow$ & PickScore $\uparrow$ & Average $\uparrow$ \\
        \midrule
    SD-1.5 & 0.2705 & 5.7726 & 0.9184 & 21.1813 & 7.0357 \\
    SFT (ours) & 0.2750  & 5.9331 & \textbf{1.3161} & 21.4159 & 7.2350  \\
    Diffusion-DPO & 0.2729 & 5.8837 & 1.1361 & 21.6064 & 7.2248 \\
    Diffusion-DPO (ours) & \underline{0.2756} & 5.8895 & 1.2219 & 21.8995 & 7.3216 \\
    \midrule
    \rowcolor{LightCyan} SPIN-Diffusion-Iter1 & 0.2739 & 6.1297 & 1.1366 & 21.6464 & 7.2967 \\
    \rowcolor{LightCyan} SPIN-Diffusion-Iter2 & 0.2751 & \underline{6.2385} & 1.3059 & \underline{22.0101} & \underline{7.4574} \\
    \rowcolor{LightCyan} SPIN-Diffusion-Iter3 & \textbf{0.2761} & \textbf{6.2769} & \underline{1.3073} & \textbf{22.0703} & \textbf{7.4827} \\
    \bottomrule
    \end{tabular}%
\end{table*}

\begin{table*}[ht!]
    \caption{The winning rate over SD-1.5 Pick-a-Pic testset. SPIN-Diffusion shows the highest winning rate over SD-1.5 among all the baselines.}
    \label{tab:results_winrate_sd}
    \centering
    \begin{tabular}{l |c c c c c c c c}
    \toprule
        Model & PickScore $\uparrow$ & HPS $\uparrow$ & ImageReward $\uparrow$ & Aesthetic $\uparrow$ & Average $\uparrow$ \\
        \midrule
    SFT (ours) & 62.4 & 82.0   & \underline{75.0}   & 70.8 & 73.2 \\
    Diffusion-DPO & 78.4 & 75.8 & 65.0   & 65.4 & 79.8 \\
    Diffusion-DPO (ours) & 83.8 & 81.2 & 71.2 & 69.0   & 84.8 \\
    \midrule
    \rowcolor{LightCyan} SPIN-Diffusion-Iter1 & 75.4 & 70.0   & 65.8 & 86.0   & 80.8 \\
    \rowcolor{LightCyan} SPIN-Diffusion-Iter2 & \underline{86.6} & \underline{82.6} & 72.6 & \underline{92.2} & \underline{90.0}   \\
    \rowcolor{LightCyan} SPIN-Diffusion-Iter3 & \textbf{87.0}   & \textbf{86.2} & \textbf{77.0}   & \textbf{93.8} & \textbf{91.6} \\
    \bottomrule
    \end{tabular}%
\end{table*}

\begin{table*}[ht!]
    \caption{The winning rate over Diffusion DPO (ours) on Pick-a-Pic testset. SPIN-Diffusion shows the highest winning rate over Diffusion DPO (ours) among all the baselines.}
    \label{tab:results_winrate_dpo}
    \centering
    \begin{tabular}{l |c c c c c c c c}
    \toprule
        Model & PickScore $\uparrow$ & HPS $\uparrow$ & ImageReward $\uparrow$ & Aesthetic $\uparrow$ & Average $\uparrow$ \\
        \midrule
    SD-1.5 & 16.2 & 20.8 & 28.8 & 31.0   & 15.2 \\
    SFT (ours) & 26.8 & 48.2 & 51.4 & 52.8 & 35.2 \\
    Diffusion-DPO & 30.6 & 29.4 & 36.8 & 45.2 & 30.4 \\
    \midrule
    \rowcolor{LightCyan} SPIN-Diffusion-Iter1 & 37.2 & 35.6 & 40.6 & 74.8 & 47.4 \\
    \rowcolor{LightCyan} SPIN-Diffusion-Iter2 & \underline{56.8} & \underline{49.0}   & \underline{52.6} & \underline{86.6} & \underline{68.2} \\
    \rowcolor{LightCyan} SPIN-Diffusion-Iter3 & \textbf{62.}4 & \textbf{56.2} & \textbf{55.8} & \textbf{86.8} & \textbf{70.2} \\
    \bottomrule
    \end{tabular}%
\end{table*}


\subsection{Additional Ablation Study}\label{subsec:ablation}

We conduct ablation study to investigate several aspects in the performance of SPIN-Diffusion.

\paragraph{Continual Training for More Epochs. }
We further study the training behavior of SPIN-Diffusion by continual training within iteration 1.
Both iteration 1 and iteration 2 commence training from the same checkpoint. However, for subsequent epochs in iteration 1, images generated by SD-1.5 are used, with SD-1.5 also serving as the opponent player. In contrast, during iteration 2, both the generated images and the opponent player originate from the iteration 1 checkpoint.
The results shown in Figure~\ref{fig:fig_limit} are reported on the $500$ prompts validation set of Pick-a-Pic. We observe that in terms of PickScore, HPS, and average score, continual training on iteration 1 even results in a performance decay. Even in terms of Aesthetic score, continual training cannot guarantee a consistent improvement. Compared to training for more epochs in iteration 1, iteration 2 has a much more ideal performance. These results show the key role in updating the opponent.

\begin{figure*}[ht!]
\centering   
\subfigure[Aesthetic]{\label{fig:limit_aes}\includegraphics[width=0.45\textwidth]{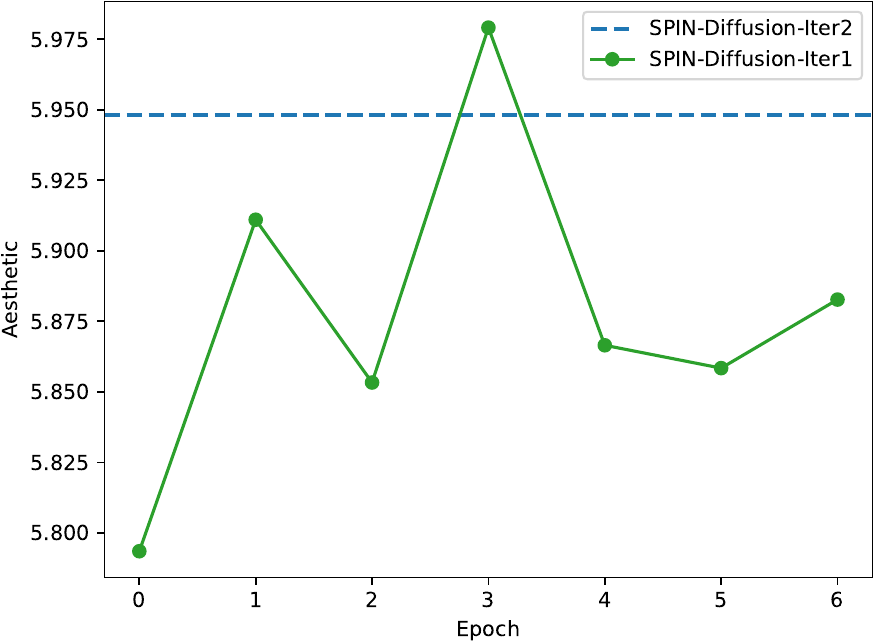}}
\subfigure[PickScore]{\label{fig:limit_pic}\includegraphics[width=0.45\textwidth]{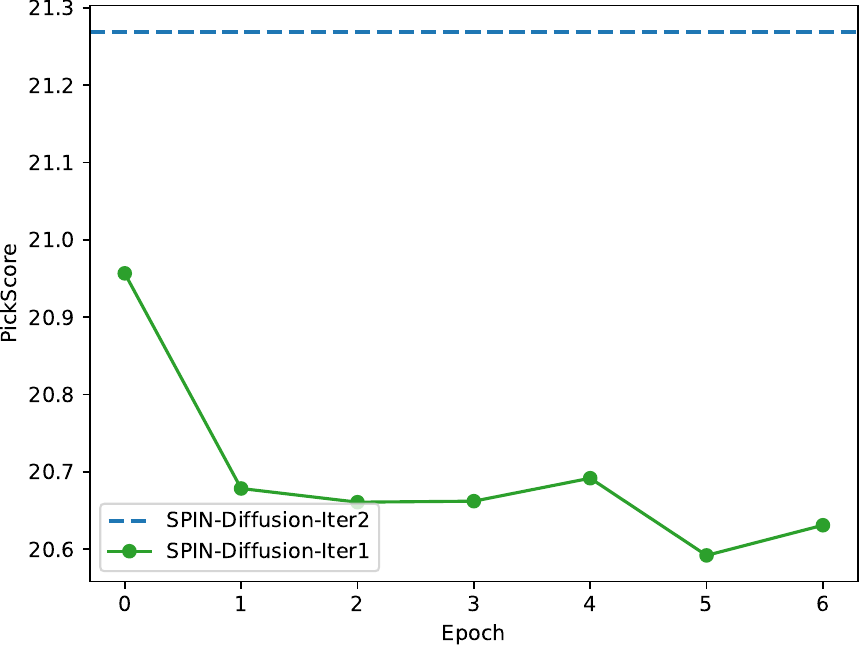}}
\subfigure[HPS]{\label{fig:limit_hps}\includegraphics[width=0.45\textwidth]{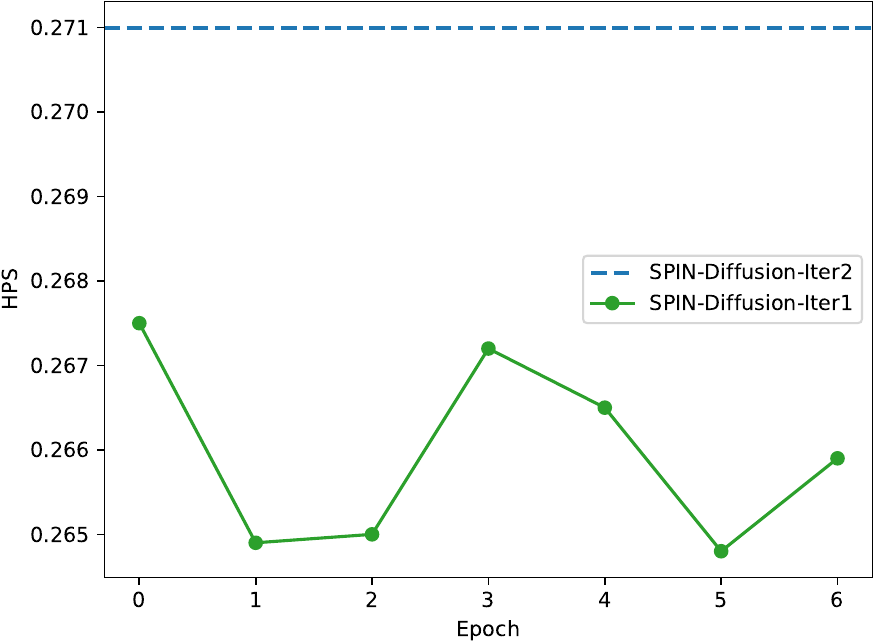}}
\subfigure[Average Score]{\label{fig:limit_avg}\includegraphics[width=0.45\textwidth]{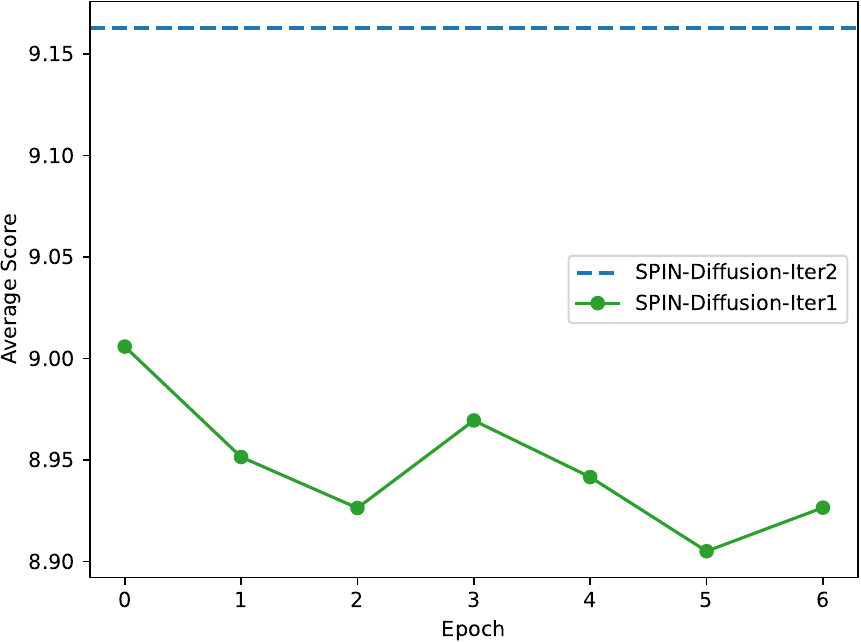}}
\caption{The evaluation results on Pick-a-Pic validation set of continual training within SPIN-Diffusion iteration 1, and SPIN-Diffusion iteration 2. The x-axis is the number of epochs. Consecutive epochs in iteration 1 reach their limit quickly while switching to iteration 2 boosts the performance.}
\label{fig:fig_limit}
\end{figure*}

\paragraph{Evaluation on Other Benchmarks}


\begin{table*}[ht!]
    \caption{The results of mean scores on PartiPrompts. We report the mean and median of PickScore, HPS, ImageReward and Aesthetic score over the whole dataset. We also report the average score over the four evaluation metrics. SPIN-Diffusion outperforms all the baselines in terms of four metrics.}
    \label{tab:parti_mean}
    \centering
    \begin{tabular}{l |c c c c c c c c}
    \toprule
        Model & HPS $\uparrow$ & Aesthetic $\uparrow$ & ImageReward $\uparrow$ & PickScore $\uparrow$ & Average $\uparrow$ \\
        \midrule
    SD-1.5 & 0.2769 & 5.6721 & 0.9196 & 21.8926 & 7.1903 \\
    SFT (ours) & 0.2814 & 5.8568 & 1.1559 & 21.9719 & 7.3165 \\
    Diffusion-DPO & 0.2815 & 5.7758 & 1.1495 & 22.2723 & 7.3698 \\
    \midrule
    \rowcolor{LightCyan} SPIN-Diffusion-Iter1 & 0.2783 & 5.9073 & 0.9952 & 22.1221 & 7.3257 \\
    \rowcolor{LightCyan} SPIN-Diffusion-Iter2 & 0.2804 & 6.0533 & 1.0845 & 22.3122 & 7.4326 \\
    \bottomrule
    \end{tabular}%
\end{table*}

\begin{table*}[ht!]
    \caption{The results of median scores on PartiPrompts. We report the mean and median of PickScore, HPS, ImageReward and Aesthetic score over the whole dataset. We also report the average score over the four evaluation metrics. SPIN-Diffusion outperforms all the baselines in terms of four metrics.}
    \label{tab:parti_median}
    \centering
    \begin{tabular}{l |c c c c c c c c}
    \toprule
        Model & HPS $\uparrow$ & Aesthetic $\uparrow$ & ImageReward $\uparrow$ & PickScore $\uparrow$ & Average $\uparrow$ \\
        \midrule
    SD-1.5 & 0.2781 & 5.6823 & 1.1247 & 21.9339 & 7.2548 \\
    SFT (ours) & 0.2781 & 5.6823 & 1.1247 & 21.9339 & 7.2548 \\
    Diffusion-DPO & \underline{0.2822} & 5.7820  & \textbf{1.3823} & \underline{22.3251} & 7.4429 \\
    \midrule
    \rowcolor{LightCyan} SPIN-Diffusion-Iter1 & 0.2793 & 5.8926 & 1.1906 & 22.1632 & 7.3814 \\
    \rowcolor{LightCyan} SPIN-Diffusion-Iter2 & 0.2810  & \underline{6.0400}   & 1.2857 & 22.2998 & \underline{7.4766} \\
    \rowcolor{LightCyan} SPIN-Diffusion-Iter3 & \textbf{0.2825} & \textbf{6.0480}  & \underline{1.3095} & \textbf{22.3361} & \textbf{7.4940} \\
    \bottomrule
    \end{tabular}%
\end{table*}

\begin{table*}[ht!]
    \caption{The results of mean scores on HPSv2. We report the mean and median of PickScore, HPS, ImageReward and Aesthetic score over the whole dataset. We also report the average score over the four evaluation metrics. SPIN-Diffusion outperforms all the baselines in terms of four metrics.}
    \label{tab:hpsv2_mean}
    \centering
    \begin{tabular}{l |c c c c c c c c}
    \toprule
        Model & HPS $\uparrow$ & Aesthetic $\uparrow$ & ImageReward $\uparrow$ & PickScore $\uparrow$ & Average $\uparrow$ \\
        \midrule
    SD-1.5 & 0.2783 & 5.9017 & 0.8548 & 21.4978 & 7.1332 \\
    SFT (ours) & 0.2846 & 6.0378 & \textbf{1.1547} & 21.8549 & 7.333  \\
    Diffusion-DPO & \underline{0.2843} & 6.0306 & \underline{1.1391} & \underline{22.3012} & 7.4388 \\
    \midrule
    \rowcolor{LightCyan} SPIN-Diffusion-Iter1 & 0.2804 & 6.1943 & 1.0133 & 21.8778 & 7.3415 \\
    \rowcolor{LightCyan} SPIN-Diffusion-Iter2 & 0.2838 & \underline{6.3403} & 1.1145 & 22.2994 & \underline{7.5095} \\
    \rowcolor{LightCyan} SPIN-Diffusion-Iter3 & \textbf{0.2849} & \textbf{6.342}  & 1.1292 & \textbf{22.3415} & \textbf{7.5244} \\
    \bottomrule
    \end{tabular}%
\end{table*}

\begin{table*}[ht!]
    \caption{The results of median scores on HPSv2. We report the mean and median of PickScore, HPS, ImageReward and Aesthetic score over the whole dataset. We also report the average score over the four evaluation metrics. SPIN-Diffusion outperforms all the baselines in terms of four metrics.}
    \label{tab:hpsv2_median}
    \centering
    \begin{tabular}{l |c c c c c c c c}
    \toprule
        Model & HPS $\uparrow$ & Aesthetic $\uparrow$ & ImageReward $\uparrow$ & PickScore $\uparrow$ & Average $\uparrow$ \\
        \midrule
    SD-1.5 & 0.2781 & 5.8529 & 0.9324 & 21.4825 & 7.1365 \\
    SFT (ours) & 0.2847 & 6.0057 & 1.308  & 21.8211 & 7.3549 \\
    Diffusion-DPO & 0.2847 & 5.9878 & 1.3085 & 22.2854 & 7.4666 \\
    \midrule
    \rowcolor{LightCyan} SPIN-Diffusion-Iter1 & 0.2803 & 6.1519 & 1.1331 & 21.858  & 7.3558 \\
    \rowcolor{LightCyan} SPIN-Diffusion-Iter2 & 0.2839 & 6.3401 & 1.2711 & 22.2577 & 7.5382 \\
    \rowcolor{LightCyan} SPIN-Diffusion-Iter3 & 0.2849 & 6.3296 & 1.2853 & 22.3029 & 7.5507 \\
    \bottomrule
    \end{tabular}%
\end{table*}

\begin{figure*}[ht!]
    \centering
    \includegraphics[width=0.99\textwidth]{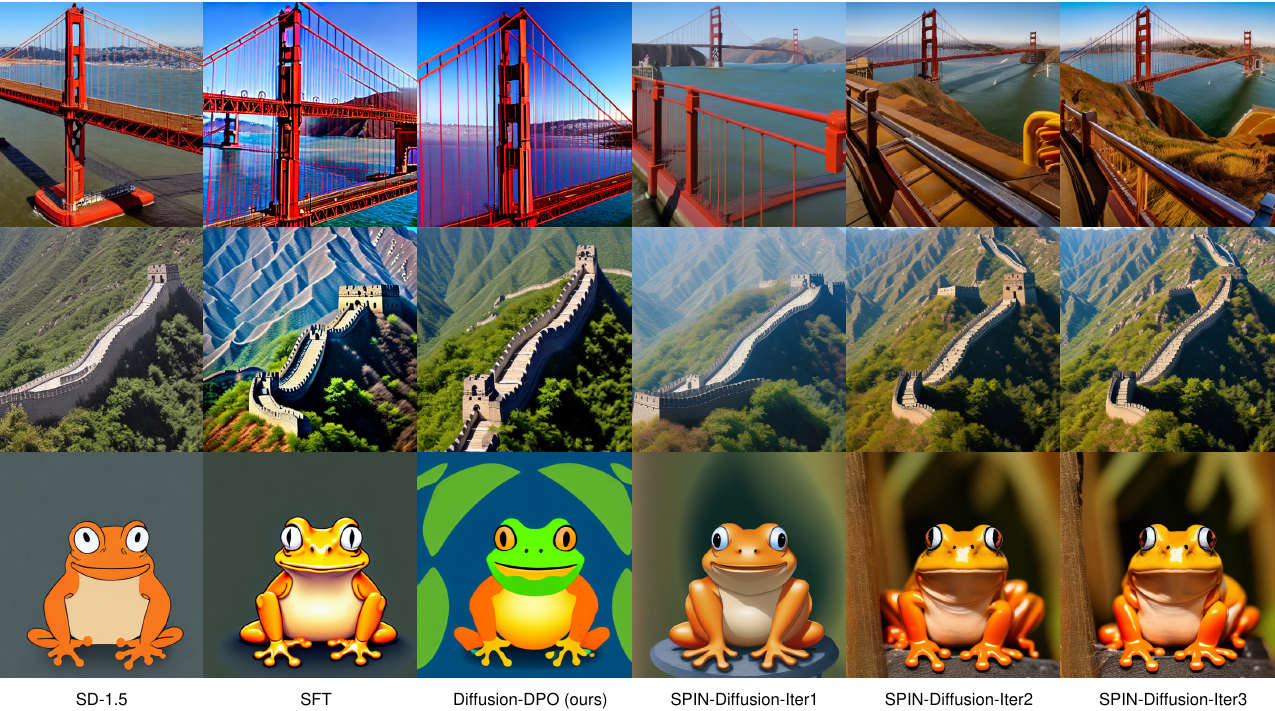}
    \caption{We show the images generated by different models based on prompts from PartiPrompts. The prompts are ``\textit{a photo of san francisco's golden gate bridge}'', ``\textit{an aerial view of the Great Wall}'' and ``\textit{Face of an orange frog in cartoon style}''. The models are: SD-1.5, SFT, Diffusion-DPO, Diffusion-DPO (ours), SPIN-Diffusion-Iter1 from left to right. SPIN-Diffusion demonstrates a notable improvement in image quality}
\label{fig:fig_qual_parti}
\end{figure*}

We also conduct experiment on PartiPrompts \citep{yu2022scaling} and HPSv2 \citep{wu2023human}. PartiPrompts consist of 1632 prompts that contains a wide range of categories and difficulties that beyond daily scenarios and natural objects. HPSv2 is a text-image prefence dataset, where the prompts come from DiffusionDB and MSCOCO \citep{lin2014microsoft} dataset. In our experiment, we use the prompts from its test set, which contains 3200 prompts. 
We use the same evaluation metrics as before and the results are shown in Table~\ref{tab:parti_mean} and~\ref{tab:parti_median}. The results show that, on both PartiPrompts and HPSv2, SPIN-Diffusion achieves a comparable performance with Diffusion DPO (ours) and surpasses other baseline models at the first iteration. SPIN-Diffusion further reaches an average score of 9.265 and 9.326 on PartiPrompts and HPSv2 dataset respectively at second iteration, which outpuerforms all other baselines by a large margin. These results consolidate our statement that $\method$ shows a superior performance over both SFT and DPO. We also conduct qualitative result on PartiPrompts and the results are shown in Figure~\ref{fig:fig_qual_parti}.


\section{Additional Details for SPIN-Diffusion}\label{app:detail}
\subsection{Additional Details of DDIM.} 
Given a prompt $\cbb$, image $\xb_0 $, sequence $\{\alpha_{t}\}_{t=1}^{T} \subseteq (0,1] $ and $\{\sigma_t\}_{t=1}^{T} \subseteq [0,+\infty)$, the forward diffusion process defined in \eqref{eq:forward} is 
\begin{align*}
q(\xb_{1: T} |\boldsymbol{x}_0):= q(\xb_T | \xb_0) \prod_{t=2}^T q(\xb_{t-1} | \xb_t, \xb_0),     
\end{align*}
where $q(\xb_T | \xb_0)= \cN(\sqrt{\alpha_T} \xb_0, (1-\alpha_T) \Ib)$ and  $q(\xb_{t-1} | \xb_t, \xb_0)$ admits the following distribution, 
\begin{align}
\cN\bigg(\sqrt{\alpha_{t-1}} \xb_0  +\sqrt{1-\alpha_{t-1}-\sigma_t^2} \cdot \frac{\xb_t-\sqrt{\alpha_t} \xb_0}{\sqrt{1-\alpha_t}}, \sigma_t^2 \Ib\bigg). \label{eq:Gaussian1} 
\end{align}
Here $\{\alpha_{t}\}_{t=1}^{T}$ is a decreasing sequence with $\alpha_0 = 1$ and $\alpha_{T}$ approximately zero. By Bayesian rule, we can show that this diffusion process ensures that $q(\xb_t | \xb_0)= \cN(\sqrt{\alpha_t} \xb_0, (1-\alpha_t) \Ib)$  for all $t$ and reduces to DDPM \citep{ho2020denoising} with a special choice of $\sigma_t = \sqrt{(1-\alpha_{t-1})/(1-\alpha_t)}\sqrt{(1-\alpha_t /\alpha_{t-1})}$. 

Given noise schedule $\alpha_t$ and $\sigma_t$, examples from the generative model follows 
\begin{align*}
p_{\btheta}(\xb_{0:T} | \cbb) &= \prod_{t=1}^{T} p_{\btheta}(\xb_{t-1} | \xb_t, \cbb) \cdot p_{\btheta}(\xb_{T} | \cbb), \notag\\
p_{\btheta}(\xb_{t-1} | \xb_t, \cbb) &= \mathcal{N}\big( \bmu_{\btheta}(\xb_t,\cbb,t), \sigma_{t}^{2} \Ib\big).    
\end{align*}
Here $\btheta$ belongs to the parameter space $\bTheta$ and $\bmu_{\btheta}(\xb_t,\cbb,t)$ is the mean of the Gaussian that can be parameterized \citep{ho2020denoising, song2020denoising} as 
\begin{align}
\bmu_{\btheta}(\xb_t,\cbb,t) &= \sqrt{\alpha_{t-1}}\bigg(\frac{\xb_t-\sqrt{1-\alpha_t} \bepsilon_{\btheta}(\xb_t, \cbb, t)}{\sqrt{\alpha_t}}\bigg)+\sqrt{1-\alpha_{t-1}-\sigma_t^2} \cdot \bepsilon_{\btheta}(\xb_t, \cbb, t), \label{eq:mutheta} 
\end{align}
where $\{\bepsilon_{\btheta}(\xb_t, \cbb, t)\}_{t=1}^{T}$ are score functions that approximate noise. Compare \eqref{eq:mutheta} and \eqref{eq:Gaussian1}, we can see that $\big(\frac{\xb_t-\sqrt{1-\alpha_t} \bepsilon_{\btheta}(\xb_t, \cbb, t)}{\sqrt{\alpha_t}}\big)$ approximates $\xb_0$, and $\bepsilon_{\btheta}$ approximates  the noise $\bepsilon_t:=\frac{\xb_t-\sqrt{\alpha_t} \xb_0}{\sqrt{1-\alpha_t}} \sim \cN(0, \Ib)$. 

\subsection{Decoupling Technique}
In Section~\ref{sec:method}, we demonstrate that the objective function defined in \eqref{eq:loss_score} can be simplified to the form in \eqref{eq:newloss}. This reformulation only requires considering two consecutive sampling steps, $t-1$ and $t$, rather than involving all intermediate steps. Now, we provide a detailed derivation.

\begin{proof}[Proof of Lemma~\ref{lm:uppderbound}]
\begin{align*}
&L_{\method}(\btheta, \btheta_k)\\
&=\EE_{\cbb \sim q(\cdot), \xb_{0:T} \sim p_{\mathrm{data}}(\cdot|\cbb), \xb'_{0:T} \sim p_{\btheta_k}(\cdot|\cbb)}\bigg[\ell\bigg(- \sum_{t=1}^{T}\frac{\beta_t}{T}\Big[\big\|\xb_{t-1} - \bmu_{\btheta}(\xb_t,\cbb,t)\big\|_{2}^{2} - \big\|\xb_{t-1} - \bmu_{\btheta_k}(\xb_t,\cbb,t)\big\|_{2}^{2}\\
&\qquad - \big\|\xb'_{t-1}- \bmu_{\btheta}(\xb'_t,\cbb,t)\big\|_{2}^{2} +\big\|\xb'_{t-1} - \bmu_{\btheta_k}(\xb'_t,\cbb,t)\big\|_{2}^{2}\Big]\bigg)\bigg]\\
&\leq \EE_{\cbb \sim q(\cdot), \xb_{0:T} \sim p_{\mathrm{data}}(\cdot|\cbb), \xb'_{0:T} \sim p_{\btheta_k}(\cdot|\cbb)}\bigg[\frac{1}{T}\sum_{t=1}^{T}\ell\bigg(- \beta_t\Big[\big\|\xb_{t-1} - \bmu_{\btheta}(\xb_t,\cbb,t)\big\|_{2}^{2} - \big\|\xb_{t-1} - \bmu_{\btheta_k}(\xb_t,\cbb,t)\big\|_{2}^{2}\\
&\qquad - \big\|\xb'_{t-1}- \bmu_{\btheta}(\xb'_t,\cbb,t)\big\|_{2}^{2} +\big\|\xb'_{t-1} - \bmu_{\btheta_k}(\xb'_t,\cbb,t)\big\|_{2}^{2}\Big]\bigg)\bigg]\\
&= \EE_{\cbb \sim q(\cdot), \xb_{0:T} \sim p_{\mathrm{data}}(\cdot|\cbb), \xb'_{0:T} \sim p_{\btheta_k}(\cdot|\cbb), t \sim \mathrm{Uniform}\{1,\ldots,T\}}\bigg[\ell\bigg(- \beta_t\Big[\big\|\xb_{t-1} - \bmu_{\btheta}(\xb_t,\cbb,t)\big\|_{2}^{2}\\
&\qquad - \big\|\xb_{t-1} - \bmu_{\btheta_k}(\xb_t,\cbb,t)\big\|_{2}^{2} - \big\|\xb'_{t-1}- \bmu_{\btheta}(\xb'_t,\cbb,t)\big\|_{2}^{2} +\big\|\xb'_{t-1} - \bmu_{\btheta_k}(\xb'_t,\cbb,t)\big\|_{2}^{2}\Big]\bigg)\bigg]\\
&= \EE_{\cbb \sim q(\cbb), (\xb_{t-1}, \xb_{t}) \sim p_{\mathrm{data}}(\xb_{t-1}, \xb_{t}|\cbb), (\xb'_{t-1}, \xb'_{t}) \sim  p_{\btheta_{k}}(\xb'_{t-1}, \xb'_{t}|\cbb), t \sim \mathrm{Uniform}\{1,\ldots,T\}.}\bigg[\ell\bigg(- \beta_t\Big[\big\|\xb_{t-1} - \bmu_{\btheta}(\xb_t,\cbb,t)\big\|_{2}^{2}\\
&\qquad - \big\|\xb_{t-1} - \bmu_{\btheta_k}(\xb_t,\cbb,t)\big\|_{2}^{2} - \big\|\xb'_{t-1}- \bmu_{\btheta}(\xb'_t,\cbb,t)\big\|_{2}^{2} +\big\|\xb'_{t-1} - \bmu_{\btheta_k}(\xb'_t,\cbb,t)\big\|_{2}^{2}\Big]\bigg)\bigg]\\
&= L_{\method}^{\mathrm{approx}}(\btheta, \btheta_k),
\end{align*}
where the first inequality is by  Jensen’s inequality and the convexity of the function $\ell$, the second equality is by integrating the average $\frac{1}{T}\sum_{t=1}^{T}$ into the expectation via $t \sim \mathrm{Uniform}\{1,\ldots, T\}$, and the third inequality holds because the argument inside the expectation is only depend of sampling step $t-1$ and $t$.  
\end{proof}

\subsection{Objective Function of SPIN-Diffusion}
We look deep into the term $\big\|\xb_{t-1} - \bmu_{\btheta}(\xb_t,\cbb,t)\big\|_{2}^{2}$ and $\big\|\xb'_{t-1} - \bmu_{\btheta}(\xb'_t,\cbb,t)\big\|_{2}^{2}$ of \eqref{eq:loss_score} and \eqref{eq:newloss} in this section.

\noindent\textbf{When $\xb_{0:T}$ Follows Forward Process.}  We have that $\xb_{0:T} \sim p_{\mathrm{data}}(\cdot|\cbb)$ and by \eqref{eq:Gaussian1} and \eqref{eq:mutheta} we have that 
\begin{align*}
\xb_{t-1} &= \sqrt{\alpha_{t-1}} \xb_0+\sqrt{1-\alpha_{t-1}-\sigma_t^2} \cdot \frac{\xb_t-\sqrt{\alpha_t} \xb_0}{\sqrt{1-\alpha_t}} + \sigma_{t} \hat{\bepsilon}_{t}\\
\bmu_{\btheta}(\xb_t,\cbb,t)&= \sqrt{\alpha_{t-1}}\bigg(\frac{\xb_t-\sqrt{1-\alpha_t} \bepsilon_{\btheta}(\xb_t, \cbb, t)}{\sqrt{\alpha_t}}\bigg)+\sqrt{1-\alpha_{t-1}-\sigma_t^2} \cdot \bepsilon_{\btheta}(\xb_t, \cbb, t),
\end{align*}
where $\hat{\bepsilon}_{t} \sim \cN(0,\Ib)$. Therefore, $\big\|\xb_{t-1} - \bmu_{\btheta}(\xb_t,\cbb,t)\big\|_{2}^{2}$ can be simplified to 
\begin{align}
h_t^2\bigg\|\frac{\xb_t-\sqrt{\alpha_t} \xb_0}{\sqrt{1-\alpha_t}} -   \bepsilon_{\btheta}(\xb_t, \cbb, t) + (\sigma_{t}/h_t)\cdot\hat{\bepsilon}_{t}\bigg\|_{2}^{2},  \label{eq:term1}
\end{align}
where $h_t = \big[\sqrt{1 - \alpha_{t-1} - \sigma_{t}^{2}} - \sqrt{\alpha_{t-1}/\alpha_t}\sqrt{1-\alpha_{t-1}}\big]$ and $\frac{\xb_t-\sqrt{\alpha_t} \xb_0}{\sqrt{1-\alpha_t}} \sim \cN(0, \Ib)$ following a Gaussian distribution. When $\sigma_t \rightarrow 0$, \eqref{eq:term1} becomes $h_t^2\big\|\bepsilon_{t} -   \bepsilon_{\btheta}(\xb_t, \cbb, t)\big\|_{2}^{2}$ with $h_t = \big[\sqrt{1 - \alpha_{t-1}} - \sqrt{\alpha_{t-1}/\alpha_t}\sqrt{1-\alpha_{t-1}}\big]$ and $\bepsilon_{t} := \frac{\xb_t-\sqrt{\alpha_t} \xb_0}{\sqrt{1-\alpha_t}} \sim \cN(0, \Ib)$.  

\paragraph{When $\xb'_{0:T}$ Follows the Backward Process.}  We have that $\xb'_{0:T} \sim p_{\btheta_{k}}(\cdot|\cbb)$ and 
\begin{align*}
\xb'_{t-1} &= \bmu_{\btheta_{k}}(\xb'_t,\cbb,t) + \sigma_{t} \hat{\bepsilon}_{t}' \\
&= \sqrt{\alpha_{t-1}}\bigg(\frac{\xb'_t-\sqrt{1-\alpha_t} \bepsilon_{\btheta_k}(\xb'_t, \cbb, t)}{\sqrt{\alpha_t}}\bigg)+\sqrt{1-\alpha_{t-1}-\sigma_t^2} \cdot \bepsilon_{\btheta_k}(\xb'_t, \cbb, t) + \sigma_{t} \hat{\bepsilon}_{t}'\\
\bmu_{\btheta}(\xb'_t,\cbb,t)&= \sqrt{\alpha_{t-1}}\bigg(\frac{\xb'_t-\sqrt{1-\alpha_t} \bepsilon_{\btheta}(\xb'_t, \cbb, t)}{\sqrt{\alpha_t}}\bigg)+\sqrt{1-\alpha_{t-1}-\sigma_t^2} \cdot \bepsilon_{\btheta}(\xb'_t, \cbb, t),
\end{align*}
where $\bepsilon_{t}' \sim \cN(0,\Ib)$. Therefore, $\big\|\xb'_{t-1} - \bmu_{\btheta}(\xb'_t,\cbb,t)\big\|_{2}^{2}$ can be simplified to 
\begin{align}
h_t^2\big\|\bepsilon_{\btheta_k}(\xb'_t, \cbb, t) -   \bepsilon_{\btheta}(\xb_t, \cbb, t) + (\sigma_t/h_t)\cdot\hat{\bepsilon}_{t}'\big\|_{2}^{2}, \label{eq:term2}
\end{align}
where $h_t = \big[\sqrt{1 - \alpha_{t-1} - \sigma_{t}^{2}} - \sqrt{\alpha_{t-1}/\alpha_t}\sqrt{1-\alpha_{t-1}}\big]$. When $\sigma_t \rightarrow 0$, \eqref{eq:term2} becomes $h_t^2\big\|\bepsilon_{\btheta_k}(\xb'_t, \cbb, t) -   \bepsilon_{\btheta}(\xb_t, \cbb, t)\big\|_{2}^{2}$ with $h_t = \big[\sqrt{1 - \alpha_{t-1}} - \sqrt{\alpha_{t-1}/\alpha_t}\sqrt{1-\alpha_{t-1}}\big]$.  

\paragraph{Simple Decoupled SPIN-Diffusion Objective Function.}
Substituting \eqref{eq:term1} and \eqref{eq:term2} into $\eqref{eq:newloss}$ and applying $\sigma_{t} \rightarrow 0$ yields, 
\begin{align}
L_{\method}^{\mathrm{approx}}(\btheta, \btheta_k) &=\EE\bigg[\ell\bigg(- \beta_t
h_t^2\Big[\big\|\bepsilon_t - \bepsilon_{\btheta}(\xb_t,\cbb,t)\big\|_{2}^{2} - \big\|\bepsilon_t - \bepsilon_{\btheta_k}(\xb_t,\cbb,t)\big\|_{2}^{2} \notag\\
&\qquad - \big\|\bepsilon_{\btheta_k}(\xb'_t,\cbb,t)- \bepsilon_{\btheta}(\xb'_t,\cbb,t)\big\|_{2}^{2} \Big]\bigg)\bigg],  \label{eq:simplifiednewloss}   
\end{align}
where $h_{t} = \sqrt{1-\alpha_{t-1}} - \sqrt{\alpha_{t-1}/\alpha_t}\sqrt{1-\alpha_{t-1}}$,  $\xb_t = \alpha_{t}\xb_0 + (1-\alpha_t)\bepsilon_t$, and the expectation is computed over the distribution,$\cbb \sim q(\cbb), \xb_{0} \sim p_{\mathrm{data}}(\xb_0|\cbb), \xb'_{t} \sim p_{\btheta_{k}}(\xb'_{t}|\cbb)$, $\bepsilon_{t} \sim \cN(0, \Ib)$ and $t \sim \mathrm{Uniform}\{1,\ldots,T\}$. \eqref{eq:simplifiednewloss} still need the intermediate steps $\xb_{t}'$, as discussed below \eqref{eq:newloss} in Section~\ref{sec:method}, we can approximate the backward process with the forward process and obtain
\begin{align*}
L_{\method}^{\mathrm{approx}}(\btheta, \btheta_k) &=\EE\bigg[\ell\bigg(- \beta_t
h_t^2\Big[\big\|\bepsilon_t - \bepsilon_{\btheta}(\xb_t,\cbb,t)\big\|_{2}^{2} - \big\|\bepsilon_t - \bepsilon_{\btheta_k}(\xb_t,\cbb,t)\big\|_{2}^{2} - \big\|\bepsilon_t'- \bepsilon_{\btheta}(\xb'_t,\cbb,t)\big\|_{2}^{2}\notag \\
&\qquad +\big\|\bepsilon_t' - \bepsilon_{\btheta_k}(\xb'_t,\cbb,t)\big\|_{2}^{2}\Big]\bigg)\bigg], 
\end{align*}
where $h_{t} = \sqrt{1-\alpha_{t-1}} - \sqrt{\alpha_{t-1}/\alpha_t}\sqrt{1-\alpha_{t-1}}$,  $\xb_t = \alpha_{t}\xb_0 + (1-\alpha_t)\bepsilon_t$,  $\xb_t' = \alpha_{t}\xb_0' + (1-\alpha_t)\bepsilon_t'$, and the expectation is computed over the distribution,$\cbb \sim q(\cbb), \xb_{0} \sim p_{\mathrm{data}}(\xb_0|\cbb), \xb'_{0} \sim p_{\btheta_{k}}(\xb'_{0}|\cbb)$, $\bepsilon_t \sim \cN(0, \Ib)$, $\bepsilon_t' \sim \cN(0, \Ib)$ and $t \sim \mathrm{Uniform}\{1,\ldots,T\}$. 

\section{Proof of Theorems in Section~\ref{sec:thm}}\label{sec:proof}

\begin{proof}[Proof of Theorem~\ref{thm:notstop}] 
We know the objective function \eqref{eq:newloss} can be simplified to $\eqref{eq:simplifiednewloss}$ by parameterize with $\bepsilon_{\btheta}$.
So we study the objective function \eqref{eq:simplifiednewloss} as follows,
\begin{align*}
L_{\method}^{\mathrm{approx}}(\btheta, \btheta_k)
&=\EE\bigg[\ell\bigg(- \beta_t
h_t^2\Big[\big\|\bepsilon_t - \bepsilon_{\btheta}(\xb_t,\cbb,t)\big\|_{2}^{2} - \big\|\bepsilon_t - \bepsilon_{\btheta_k}(\xb_t,\cbb,t)\big\|_{2}^{2}\\
&\qquad - \big\|\bepsilon_{\btheta_k}(\xb'_t,\cbb,t)- \bepsilon_{\btheta}(\xb'_t,\cbb,t)\big\|_{2}^{2} \Big]\bigg)\bigg].
\end{align*}
Since $\btheta_{k}$ is not the global optimum of $L_{\mathrm{DSM}}$, there exists $\btheta^{*}$ such that $L_{\mathrm{DSM}}(\btheta^{*}) \leq L_{\mathrm{DSM}}(\btheta_k)$, which gives that 
\begin{align}
 \mathbb{E}\Big[\gamma_t\big\| \bepsilon_{\btheta^{*}}(\xb_t, \cbb, t) - \bepsilon_t\big\|_2^2\Big] \leq  \mathbb{E}\Big[\gamma_t\big\| \bepsilon_{\btheta_k}(\xb_t, \cbb, t) - \bepsilon_t\big\|_2^2\Big], \label{eq:notoptimal}
\end{align}
where the expectation is computed over the distribution $\cbb\sim q(\cdot), \xb_0 \sim q_{\mathrm{data}}(\cdot|\cbb), \bepsilon_t\sim \cN(0, \Ib)$, $t \sim \mathrm{Uniform}\{1,\ldots,T\}$. Define $g(s) = L_{\method}^{\mathrm{approx}}(\btheta^{*}, \btheta_k)$ with $\beta_{t} = s\gamma_{t}/h_{t}^2$ as follows,
\begin{align*}
g(s) &= \EE\bigg[\ell\bigg(- \beta_t
h_t^2\Big[\big\|\bepsilon_t - \bepsilon_{\btheta^{*}}(\xb_t,\cbb,t)\big\|_{2}^{2} - \big\|\bepsilon_t - \bepsilon_{\btheta_k}(\xb_t,\cbb,t)\big\|_{2}^{2} - \big\|\bepsilon_{\btheta_k}(\xb'_t,\cbb,t)- \bepsilon_{\btheta^{*}}(\xb'_t,\cbb,t)\big\|_{2}^{2} \Big]\bigg)\bigg]\\
&= \EE\bigg[\ell\bigg(- s\lambda_t\Big[\big\|\bepsilon_t - \bepsilon_{\btheta^{*}}(\xb_t,\cbb,t)\big\|_{2}^{2} - \big\|\bepsilon_t - \bepsilon_{\btheta_k}(\xb_t,\cbb,t)\big\|_{2}^{2} - \big\|\bepsilon_{\btheta_k}(\xb'_t,\cbb,t)- \bepsilon_{\btheta^{*}}(\xb'_t,\cbb,t)\big\|_{2}^{2} \Big]\bigg)\bigg].
\end{align*}
Then we have that $g(0) = 0$ and
\begin{align*}
\frac{dg}{ds}(0) &= \EE\bigg[-\ell'(0)\lambda_t\bigg(\big\|\bepsilon_t - \bepsilon_{\btheta^{*}}(\xb_t,\cbb,t)\big\|_{2}^{2} - \big\|\bepsilon_t - \bepsilon_{\btheta_k}(\xb_t,\cbb,t)\big\|_{2}^{2} - \big\|\bepsilon_{\btheta_k}(\xb'_t,\cbb,t)- \bepsilon_{\btheta^{*}}(\xb'_t,\cbb,t)\big\|_{2}^{2}\bigg)\bigg]\\
&= -\ell'(0) \bigg(\EE\gamma_t\big\|\bepsilon_t - \bepsilon_{\btheta^{*}}(\xb_t,\cbb,t)\big\|_{2}^{2} - \EE\gamma_t\big\|\bepsilon_t - \bepsilon_{\btheta_k}(\xb_t,\cbb,t)\big\|_{2}^{2}\\
&\qquad - \EE\gamma_t\big\|\bepsilon_{\btheta_k}(\xb'_t,\cbb,t)- \bepsilon_{\btheta^{*}}(\xb'_t,\cbb,t)\big\|_{2}^{2}\bigg)\\
&<0,
\end{align*}
where the last inequality is by \eqref{eq:notoptimal}. Here $\xb_t = \sqrt{\alpha_t}\xb_0 + \sqrt{1-\alpha_t}\bepsilon_t$ and the expectation is computed over the distribution $\cbb\sim q(\cdot), \xb_0 \sim q_{\mathrm{data}}(\cdot|\cbb), \bepsilon_t\sim \cN(0, \Ib)$, $t \sim \mathrm{Uniform}\{1,\ldots,T\}$.

Therefore, there exist a $\lambda_0$ such that for all $0 < \lambda < \lambda_0$, we have $g(\lambda) < \ell(0)$. So for those $\beta_{t} = s\gamma_{t}/h_{t}^2$ with $0 < \lambda < \lambda_0$, we have that 
\begin{align*}
L_{\method}^{\mathrm{approx}}(\btheta^{*}, \btheta_{k}) = g(\lambda) < g(0) = L_{\method}(\btheta_{k}, \btheta_{k}),
\end{align*}
where the inequality holds due to the choice of $\lambda$. Therefore, we conclude that $\btheta_k$ is not the global optimum of \eqref{eq:newloss}.
\end{proof}

\begin{proof}[Proof of Theorem~\ref{thm:stop}]
By \eqref{eq:newloss} we have that,
\begin{align*}
L_{\method}^{\mathrm{approx}}(\btheta, \btheta_k)&=\EE\bigg[\ell\bigg(- \beta_t\Big[\big\|\xb_{t-1} - \bmu_{\btheta}(\xb_t,\cbb,t)\big\|_{2}^{2}  - \big\|\xb_{t-1} - \bmu_{\btheta_k}(\xb_t,\cbb,t)\big\|_{2}^{2}\notag \\
&\qquad- \big\|\xb'_{t-1}- \bmu_{\btheta}(\xb'_t,\cbb,t)\big\|_{2}^{2}+\big\|\xb'_{t-1} - \bmu_{\btheta_k}(\xb'_t,\cbb,t)\big\|_{2}^{2}\Big]\bigg)\bigg], 
\end{align*}
where the expectationis computed over the distribution $\cbb \sim q(\cbb)$,  $(\xb_{t-1}, \xb_{t}) \sim \int p_{\mathrm{data}}(\xb_{0}|\cbb)q(\xb_{t-1}, \xb_{t}|\xb_{0})d\xb_0$, $(\xb_{t-1}', \xb'_{t}) \sim \int p_{\btheta_{k}}(\xb'_{0}|\cbb)q(\xb'_{t-1}, \xb'_{t}|\xb'_{0})d\xb'_0$, $t \sim \mathrm{Uniform}\{1,\ldots,T\}$. Since $p_{\mathrm{data}}(\cdot|\cbb)= p_{\btheta_t}(\cdot|\cbb)$, we can conclude that $(\xb_{t-1}, \xb_{t})$ and  $(\xb'_{t-1}, \xb'_{t})$ are independent and identically distributed random variable. Therefore, by symmetry property of $(\xb_{t-1}, \xb_t)$ and $(\xb_{t-1}', \xb'_t)$, we have for any $\btheta \in \bTheta$ that 

\begin{align*}
L_{\method}^{\mathrm{approx}}(\btheta, \btheta_k)&=\frac{1}{2}\EE\bigg[\ell\bigg(- \beta_t\Big[\big\|\xb_{t-1} - \bmu_{\btheta}(\xb_t,\cbb,t)\big\|_{2}^{2}  - \big\|\xb_{t-1} - \bmu_{\btheta_k}(\xb_t,\cbb,t)\big\|_{2}^{2}\\
&\qquad- \big\|\xb'_{t-1}- \bmu_{\btheta}(\xb'_t,\cbb,t)\big\|_{2}^{2}+\big\|\xb'_{t-1} - \bmu_{\btheta_k}(\xb'_t,\cbb,t)\big\|_{2}^{2}\Big]\bigg)\\
&\qquad+\ell\bigg(- \beta_t\Big[\big\|\xb'_{t-1} - \bmu_{\btheta}(\xb'_t,\cbb,t)\big\|_{2}^{2}  - \big\|\xb'_{t-1} - \bmu_{\btheta_k}(\xb'_t,\cbb,t)\big\|_{2}^{2}\\
&\qquad- \big\|\xb_{t-1}- \bmu_{\btheta}(\xb_t,\cbb,t)\big\|_{2}^{2}+\big\|\xb_{t-1} - \bmu_{\btheta_k}(\xb_t,\cbb,t)\big\|_{2}^{2}\Big]\bigg)\bigg]\\
&\geq \EE\bigg[\ell\bigg(- \frac{\beta_t}{2}\Big[\big\|\xb_{t-1} - \bmu_{\btheta}(\xb_t,\cbb,t)\big\|_{2}^{2}  - \big\|\xb_{t-1} - \bmu_{\btheta_k}(\xb_t,\cbb,t)\big\|_{2}^{2}\\
&\qquad- \big\|\xb'_{t-1}- \bmu_{\btheta}(\xb'_t,\cbb,t)\big\|_{2}^{2}+\big\|\xb'_{t-1} - \bmu_{\btheta_k}(\xb'_t,\cbb,t)\big\|_{2}^{2}\Big]\\
&\qquad - \frac{\beta_t}{2}\Big[\big\|\xb'_{t-1} - \bmu_{\btheta}(\xb'_t,\cbb,t)\big\|_{2}^{2}  - \big\|\xb'_{t-1} - \bmu_{\btheta_k}(\xb'_t,\cbb,t)\big\|_{2}^{2}\\
&\qquad- \big\|\xb_{t-1}- \bmu_{\btheta}(\xb_t,\cbb,t)\big\|_{2}^{2}+\big\|\xb_{t-1} - \bmu_{\btheta_k}(\xb_t,\cbb,t)\big\|_{2}^{2}\Big]\bigg)\bigg]\\
&= \ell(0),
\end{align*}
where the inequality is due to Jensen's inequality (recalling that $\ell$ is convex in Assumption \ref{assm:1}), and  the expectation is computed over the distribution $\cbb \sim q(\cbb)$,  $(\xb_{t-1}, \xb_{t}) \sim \int p_{\mathrm{data}}(\xb_{0}|\cbb)q(\xb_{t-1}, \xb_{t}|\xb_{0})d\xb_0$, $(\xb_{t-1}', \xb'_{t}) \sim \int p_{\btheta_{k}}(\xb'_{0}|\cbb)q(\xb'_{t-1}, \xb'_{t}|\xb'_{0})d\xb'_0$, $t \sim \mathrm{Uniform}\{1,\ldots,T\}$. Therefore, we have that 
\begin{align*}
L^{\mathrm{approx}}_{\method}(\btheta, \btheta_k) \geq \ell(0) = L^{\mathrm{approx}}_{\method}(\btheta_k, \btheta_k),    
\end{align*}
which means that $\btheta_k$ is the global optimum of \eqref{eq:newloss}. As a consequence, $\btheta_{k+1} = \btheta_{k}$.  
\end{proof}

\bibliography{deeplearningreference,selftraining}
\bibliographystyle{ims}

\end{document}